\documentclass[a4paper, 10pt]{article}

\usepackage{amsthm} 
\usepackage{amsmath}
\usepackage{amssymb}
\usepackage{mathtools}
\usepackage{amsfonts}
\usepackage{lipsum}
\usepackage{graphicx}
\usepackage{epstopdf}
\usepackage{algorithmic}
\usepackage[natbib=true, style=authoryear-comp, maxbibnames=14, maxcitenames=1, uniquelist=false, citetracker=true, backend=biber, sortcites=false, doi=false, url=false, firstinits=true]{biblatex}
\DeclareNameAlias{sortname}{family-given}
\AtEveryCitekey{\ifciteseen{}{\defcounter{maxnames}{2}}}
\usepackage{geometry}
\usepackage{microtype}
\usepackage{graphicx}
\usepackage{booktabs} 
\usepackage{authblk}

\bibliography{bibliography.bib,stats.bib}

\ifpdf
  \DeclareGraphicsExtensions{.eps,.pdf,.png,.jpg}
\else
  \DeclareGraphicsExtensions{.eps}
\fi
\usepackage{subfig}

\usepackage[colorlinks=True, citecolor=blue]{hyperref}

\usepackage[capitalize,noabbrev]{cleveref}

\usepackage{enumitem}
\setlist[enumerate]{leftmargin=.5in}
\setlist[itemize]{leftmargin=.5in}

\theoremstyle{plain}
\newtheorem{theorem}{Theorem}[section]
\newtheorem{proposition}[theorem]{Proposition}
\newtheorem{lemma}[theorem]{Lemma}
\newtheorem{corollary}[theorem]{Corollary}
\theoremstyle{definition}
\newtheorem{definition}[theorem]{Definition}
\newtheorem{assumption}[]{Assumption}
\theoremstyle{remark}

\author[1]{Martin Rouault\thanks{Corresponding author: \href{mailto: rouault.martin@protonmail.com}{rouault.martin@protonmail.com}}}
\author[1]{Rémi Bardenet}
\author[2]{Mylène Ma\"ida}
\affil[1]{\small Univ. Lille, CNRS, Centrale Lille, UMR 9189 -- CRIStAL, 59651 Villeneuve d'Ascq, France}
\affil[2]{\small Univ. Lille, CNRS, UMR 8524 -- Laboratoire Paul Painlevé, F-59000 Lille, France}

\usepackage{amsopn}

\def\PP{\mathbb{P}_{n,\beta_n}^V}

\renewcommand\d{\mathrm{d}}
\newcommand\mydots{\makebox[1em][c]{.\hfil.\hfil.\hfil}}
\def\threefig{0.32\textwidth}
\def\twofig{0.5\textwidth}
\newcommand\rev[1]{#1}

\title{Monte Carlo with kernel-based Gibbs measures:\\ Guarantees for probabilistic herding.}

\begin{document}

\maketitle
\begin{abstract}
    Kernel herding belongs to a family of deterministic quadratures
    that seek to minimize the maximum mean discrepancy (MMD), that is, 
    the worst-case integration error over a reproducing kernel Hilbert space (RKHS). 
    These MMD minimization procedures come with strong experimental support, but comparatively less theoretical footing.
    In particular, apart from recent progress in distribution compression, little has been proved in favor of an improvement of MMD minimization over classical Monte Carlo quadrature when the RKHS is infinite-dimensional. 
  In this paper, we study a joint probability distribution over quadrature nodes, a tailored \emph{Gibbs distribution}, whose support intuitively tends to concentrate around MMD minimizers as a temperature parameter is decreased. 
  Our main contribution is to prove that drawing integration nodes from our distribution does outperform i.i.d Monte Carlo. 
  While our bounds on the worst-case integration error feature the same rate as i.i.d. Monte Carlo, we do obtain a tighter concentration inequality as the temperature parameter decreases.
  This means smaller confidence intervals as the number of quadrature nodes increases.
  While arguably a first step, our results demonstrate that the mathematical toolbox developed around Gibbs measures can help understand to what extent kernel herding and its variants improve on computationally cheaper methods.
  There remains the issue of sampling from our Gibbs distribution. 
  In our numerical experiments, we demonstrate that a simple MCMC chain already yields approximate samples that lead to improved confidence intervals around the target integrals, as supported by our theoretical results.  
    
\end{abstract}


    \emph{MSCcodes:} 65C05, 68W20, 60E15, 82B44, 82B31.

\section{Introduction}
\label{s:introduction}

Numerical integration with respect to a possibly unnormalized distribution $\pi$ on $\mathbb{R}^{d}$ is routine in computational statistics \citep{robert2007bayesian} and probabilistic machine learning \citep{murphy2023probabilistic}.
Monte Carlo algorithms \citep{robert2004monte} are randomized algorithms that tackle this task, defining estimators that rely on $n$ evaluations of the integrand $f$ at suitably chosen random points in $\mathbb{R}^d$, called \emph{nodes}. 
Classical algorithms, such as Markov Chain Monte Carlo (MCMC), come with probabilistic error controls, such as central limit theorems, that involve errors of magnitude $1/\sqrt{n}$.
The popularity of MCMC has justified a continuous research effort to improve on that rate, which is too slow when evaluating the integrand is computationally costly.
This happens, e.g., in Bayesian inference of models involving large systems of differential equations, or Bayesian learning of complex Markov models. 
Indeed, the time needed for one evaluation of the likelihood can range from a few minutes in cardiac electrophysiology models \citep{johnstone2016uncertainty} to a few hours for population models in ecology \citep{purves2013biology}.
Such applications call for Monte Carlo methods that make the most of the $n$ likelihood or integrand evaluations, and one prefers spending a larger computational budget to design better quadrature nodes to increasing $n$.
Quasi-Monte Carlo methods (QMC), for instance, rely on smoothness assumptions to obtain a worst-case error of order $1/n$. 
A common smoothness assumption is that the target integrands $f$ belong to a particular reproducing kernel Hilbert space (RKHS); see e.g. \citep[Section 3]{dick2013qmc}.

At another end of the algorithmic spectrum, variational Bayesian methods (VB; \citealp{blei2017variational}), such as Stein variational gradient descent \citep{liu2016stein}, sacrifice some of the error controls to gain in scalability.
At its core, VB is the minimization of a dissimilarity measure between a candidate approximation and the target distribution $\pi$. 
Minimizing a relative entropy, for instance, yields algorithms amenable to stochastic gradient techniques \citep{hoffman2013stochastic}, yet that usually come with at best loose theoretical guarantees on how well integrals w.r.t. $\pi$ are approximated.

An intermediate method between Monte Carlo and relative entropy-based VB is the minimization of an integral probability metric (IPM) of the form
\begin{equation}
    \label{e:IPM}
    \nu\mapsto I_{K}(\nu - \pi) = \iint K(x, y)\,\mathrm{d}(\nu - \pi)^{\otimes 2}(x, y),
\end{equation}
where $K$ is a positive definite kernel, known as the \emph{interaction kernel}.
It is known \citep{gretton2010hilbert} that the square root of $I_{K}(\nu - \pi)$ in \eqref{e:IPM} is the worst-case integration error for integrands in the unit ball of the RKHS defined by $K$, when approximating $\pi$ by $\nu$; see \citep{pronzato2020bayesian}.
Loosely speaking, minimizing \eqref{e:IPM} is thus an attempt at designing efficient algorithms in the vein of VB, yet that come with a control on the integration error like Monte Carlo and QMC.

Kernel herding, for instance, which rose to attention in the context of learning Markov random fields \citep{welling2012herding, welling2012herding_part, chen2010parametric, bach2012herding}, is a conditional gradient descent that greedily minimizes \eqref{e:IPM}.
One practical limitation is the requirement to evaluate the kernel embedding 
$
    \int K(\cdot, x)\,\mathrm{d}\pi(x).
$
When the support of the target measure $\pi$ is all of $\mathbb{R}^d$, this requirement can be circumvented by a new choice of kernel based on $K$ and $\pi$, namely the \emph{Stein kernel} $K_{\pi}$ \citep{anastasiou2023stein}.
The IPM \eqref{e:IPM} for $K_{\pi}$ coincides with the \emph{kernel Stein discrepancy} (KSD), and gradient descent yields efficient minimization algorithms \citep{korba2021kernel}.
Modifying the kernel comes however at the price of modifying in a nontrivial way the set of integrands involved in the worst-case guarantee, which might be an issue in applications \citep{benard2024kernel}.

Besides the algorithmic limitation of the evaluation of the kernel embedding, the main limitation of IPM minimization algorithms has been theoretical: 
    there are few results that support an improvement over classical Monte Carlo in terms of the worst-case integration error, at least with assumptions general enough to cover practical use.
For instance, there is experimental support in favor of an $n^{-1}$ convergence rate of the worst-case integration error in the RKHS induced by $K$ for the $n-$point output of the algorithm \citep{chen2010super, pronzato2023performance}, yet proving an $o(n^{-1/2})$ decay of the rate in the practically relevant case of an infinite-dimensional RHKS is still rather open.
To our knowledge, the clearest results in that direction are for a two-stage procedure called \emph{kernel thinning} \citep{dwivedi2021generalized, dwivedi2024kernel,shetty2022distributioncompressionnearlineartime} and currently summarize as follows. 
The construction draws $n^2$ points from a well-mixing Markov chain targeting $\pi$, then subsamples the chain's history down to $n$ points in time $\mathcal{O}(n^2)$, up to log terms, to obtain an $n$-point supported measure.
    The worst-case integration error then depends on the kernel, and has been shown to as a small as $\mathcal{O}_{\mathbb{P}}(n^{-1}\log^{d+1} n)$ for Gaussian and inverse multi-quadratic kernels.

In this paper, we initiate a different line of attack 
    towards proving that imposing low values of an IPM yields to more efficient quadratures than classical Monte Carlo.
We propose to relax the objective of minimizing a worst-case integration error and rather sample from a probability distribution that tends to concentrate around the minimizing configurations of $n$ points.
Technically, we propose to use the formalism of Gibbs measures from statistical physics, which comes with a large mathematical toolbox. 
Our motivation is to investigate to what extent these techniques can help analyze (a probabilistic relaxation of) minimizing \eqref{e:IPM}.


Gibbs measures are probability distributions that describe systems of interacting particles.
Choosing the interaction carefully, one can arrange the Gibbs measure to favor configurations of points that tend to minimize $I_{K}$ in \eqref{e:IPM}, when a suitable inverse temperature parameter goes to infinity. 
Gibbs measures have been studied for decades in probability and mathematical physics, with a focus on models that relate to electromagnetism \citep{serfaty2018systems}.
Classical results include large deviation principles (LDPs; \citealp{chafai2014first}) and concentration inequalities in some cases \citep{chafai2018concentration}.
Our main contribution is to prove that the Gibbs measure whose points \emph{repel} each other by an amount given by a \emph{bounded} kernel $K(x,y)$ satisfies a concentration inequality for the worst-case integration error~\eqref{e:IPM}; see Theorem~\ref{th:concentration_inequality}.
Our Corollary~\ref{c:sub_Gaussian_decay} shows a faster sub-Gaussian decay than the i.i.d or MCMC case, coming from the repulsion.
In other words, our probabilistic relaxation of herding provably outperforms classical Monte Carlo methods, in the sense that it comes with tighter confidence intervals for a fixed worst-case integration error~\eqref{e:IPM}.
We emphasize that the size of the confidence intervals we provide still decays as $n^{-1/2}$, like classical Monte Carlo, but the coverage is improved through one of the dependences in $n$ being replaced by the inverse temperature.

There are many limitations to be studied in future work, however. 
In particular, in the absence of a better sampling algorithm, we currently rely on MCMC to sample from our Gibbs measure; while we provide experimental evidence that the statistical properties of the target are well approximated, mixing time results tailored to our particular target would be important. 
Moreover, we assume that the target measure has compact support and rely on an approximation to evaluate the kernel embedding, as in most herding papers.
Here again, we follow usage and numerically assess the impact of that approximation step, but incorporating an estimator of the kernel embedding in the theory would be a plus, and it would demonstrate another interest of Gibbs measures in the analysis of kernel-based methods.

The rest of the paper is organized as follows. 
In Section~\ref{s:related_work}, we survey worst-case controls on the integration error.
In Section~\ref{s:results}, we introduce a family of Gibbs measures and state our theoretical results.
In Section~\ref{s:experiments}, we explain how to approximately sample from such Gibbs measures, and experimentally validate our claims.
We discuss perspectives in Section~\ref{s:discussion} and gather proofs in Section~\ref{s:proofs}.
Additional experiments and discussions mentioned in Section~\ref{s:experiments} are in the appendix.

\section{Related work}
\label{s:related_work}


\paragraph{Uniform concentration for Monte Carlo}
The introduction of a new Monte Carlo method is typically backed up by a central limit theorem \citep{robert2004monte}.
In practice, where the number $n$ of quadrature nodes is fixed, one prefers a concentration inequality, to derive a confidence interval for $\int f \d\pi$.
While rarely put forward, many applications further require a uniform control over several integrands. 
For instance, in multi-class classification with $0/1$ loss and $M$ classes, determining the Bayesian predictor involves giving a joint confidence region over $M-1$ integrals.
This motivates studying the simultaneous approximation of several integrals by a single set of $n$ Monte Carlo nodes.
One way to formalize this problem is by upper bounding the Wasserstein distance
\begin{equation}
    \label{e:wasserstein}
    W_{1}(\mu_n, \pi) = \underset{\lVert f\rVert_{\mathrm{Lip}} \leq 1}{\sup}\left\lvert \int f\,\mathrm{d}(\mu_n - \pi)\right\rvert,
\end{equation}
where $\mu_n$ is the Monte Carlo empirical approximation of the target measure $\pi$, 
and the supremum is taken over all $1$-Lipschitz functions. 
\citep{bolley2007quantitative, fournier2015rate} give a concentration inequality for \eqref{e:wasserstein} when $\mu_n$ is the empirical measure of i.i.d. draws from $\pi$.
\begin{theorem}[\citet{bolley2007quantitative}]
    \label{th:bolley}
    If $x_1, \dots, x_n \sim \pi$, under a suitable moment assumption on $\pi$, for any $d' > d$, there exists $n_0$ such that, for any $n \geq n_0$ and $r > n^{-1/(d' +2)}$,
    \begin{equation}
            \label{e:concentration_iid}
            \mathbb{P}\left[W_1(\mu_n, \pi) > r\right] \leq \exp\left(- \alpha n r^2\right),
    \end{equation}
    where $\alpha$ is a constant depending on $\pi$ and $\mu_n = n^{-1}\sum_{i = 1}^{n}\delta_{x_i}$.
    In particular, fix $\epsilon > 0$ and $\delta \in (0, 1)$.
    For any $n$ such that $n^{-1/(d'+2)} \leq \epsilon$ and $\exp\left(-\alpha n \epsilon^{2}\right) \leq \delta$,
    with probability higher than $1 - \delta$,
    \begin{equation}
            \label{e:quadrature_concentration_iid}
            \underset{\lVert f\rVert_{\mathrm{Lip}} \leq 1}{\sup}\left\lvert \int f\,\mathrm{d}(\mu_n - \pi)\right\rvert \leq \epsilon.
    \end{equation}
\end{theorem}
Two comments are in order. 
First, analogous bounds with the same rate exist for Markov chains \citep{fournier2015rate}.
Second, if we take the supremum over smoother functions, like the functions in the RKHS $\mathcal{H}_K$ \citep{BeTh11} of positive definite kernel $K,$
one can further hope to reach inequality \eqref{e:quadrature_concentration_iid}  for smaller values of $n$.

In fact, it is known \citep{dwivedi2024kernel, li2024debiaseddistributioncompression} that the worst-case integration error in $\mathcal{H}_K$ is $\mathcal{O}_{\mathbb{P}}(n^{-1/2})$ for geometrically ergodic Markov chains.
Moreover, \citep{tolstikhin2016minimax} have shown that this $n^{-1/2}$ is optimal for i.i.d points.

\paragraph{Variational Bayes and kernel herding}
Convergence guarantees for VB are often formulated in terms of the minimized dissimilarity measure; see \citep{alquier2016properties, lambert2022variational} and references therein. 
For instance, under strong assumptions on the target $\pi$ and the allowed variational approximation, \citep{lambert2022variational} give rates for the convergence to the minimal achievable relative entropy $\mathrm{KL}(\rho_t || \pi)$ between $\pi$ and the $t-$th iterate $\rho_t$ of an idealized (continuous-time) VB algorithm. 
For Stein variational gradient descent (SVGD), \citet{korba2020svgd} prove a decay of $\mathrm{KL}(\rho_t || \pi)$ along with non-asymptotic bounds at rate $t^{-1}$ for the kernel Stein discrepancy (KSD) between an idealized (continuous version) of the $t-$th iterate of the algorithm and the target $\pi$.
More recently, a noisy version of SVGD has been studied in \citep{priser2024longtimeasymptoticsnoisysvgd}, where the authors show the convergence this time for the actual $n-$point empirical measure ouput by these noisy SVGD updates when $n$ is fixed at first and the number of iterations $t$ goes to infinity.
Moreover, \citet{shi2024finite} proved a decay at rate $(\log\log n)^{-1/2}$ and depending on $t$ for the true $n-$point SVGD updates after $t$ iterations; it is to our knowledge the first rate of convergence result for the true SVGD algorithm.

Since the KSD is a particular case of \eqref{e:IPM}, this implies a control on a worst-case integration error.
Indeed, if $K$ is positive definite, so that it defines an RKHS $\mathcal{H}_{K}$, it can be shown that $I_K(\nu-\pi)$ in \eqref{e:IPM} is the square of the worst-case integration error over the unit ball of $\mathcal{H}_{K}$ when replacing $\pi$ by $\nu$; see e.g. \citep{gretton2010hilbert} and Proposition~\ref{p:dual}.
For KSD however, the RKHS corresponds to the Stein kernel, and is different from the original $\mathcal{H}_{K}$ in a way that is not yet fully understood \citep{benard2024kernel}.
It has nonetheless been shown \citep{gorham2017measuring, barp2022targeted, kanagawa2024controllingmomentskernelstein} that KSD or its variants control bounded-Lipschitz, Wasserstein or q-Wasserstein convergence, along with bounds on these distances in terms of the KSD. 
Yet to our knowledge the rate of convergence of KSD does not transfer to rates on the original $\mathcal{H}_K$.

Under some assumptions on $\mathcal{H}_K$, kernel herding algorithms have been proved to achieve $I_{K}(\mu_n - \pi) \leq c n^{-2}$, so that the worst-case quadrature error decreases at rate $n^{-1}$ \citep{chen2010super}. 
Other variants of conditional gradient algorithms led to further improvement, up to convergences of the type $I_{K}(\mu_n - \pi) \leq \exp\left(-c n\right)$  \citep{bach2012herding}.
While those assumptions are reasonable when the dimension of $\mathcal{H}_K$ is finite, \citep{bach2012herding} have shown that they are \emph{never} fulfilled in the infinite-dimensional setting.
In that case, the only general result is the ``slow" rate $I_{K}(\mu_n - \pi) \leq c n^{-1}$ \citep{bach2012herding}.
Variants of kernel herding have also led to the same slow rate, sometimes up to a logarithmic term \citep{chen2018steinpoints, chen2019steinmcmc, mak2018support}.

More recently, in a parallel line of work called \emph{distribution compression}, it has been shown \citep{dwivedi2024kernel,dwivedi2021generalized} that if one is allowed to draw $n^2$ pointbelts from a well-mixing Markov chain converging to $\pi$, it is possible to subsample the chain's history down to $n$ nodes, with worst-case integration error in the RKHS in $\mathcal{O}_{\mathbb{P}}(n^{-1} \log^{3d/4+1} n)$ for a Gaussian kernel, or $\mathcal{O}_{\mathbb{P}}(n^{-1} \log^{d+1}n)$ for inverse multi-quadratic kernels.
Moreover, the subsampling step can be done in time roughly $\mathcal{O}(n^2)$
\citep{shetty2022distributioncompressionnearlineartime}, and more arbitrary inputs can even be considered in the first stage \citep{li2024debiaseddistributioncompression}.
While the $n^{-1}$ in the rate is very appealing, the logarithmic power means that the bound is interesting only when $n$ is at least exponential in $d$, which is prohibitive when the integrand is expensive to evaluate.
A stronger rate has been achieved in \citep{dwivedi2021generalized}, but weakening the objective to integrating a \emph{single} integrand $f$ in a given RKHS $\mathcal{H}_K$.
In that case, the authors of \citep{dwivedi2021generalized} obtained a confidence interval of size $\mathcal{O}_{\mathbb{P}}(n^{-1}\sqrt{\log n})$, a quadratic improvement over classical Monte Carlo.

Our long-term goal is to match the cost and performance of these recent advances, while providing more practical confidence intervals, for instance with a central limit theorem with tractable asymptotic variance. 
To that end, we design a Gibbs measure that relaxes the objective of herding, and directly outputs an $n$-point supported measure to which we can apply a large toolbox of results developed in statistical physics. 
This paper is a first step in our program towards a central limit theorem for this Gibbs measure, and we focus here on concentration for the worst-case error.

\paragraph{Optimal kernel quadrature}
    In another line of research, optimal kernel quadrature has been shown to provide bounds for \eqref{e:IPM}.
    To cite a recent result,  the authors of \citep{BeBaCh20} give a bound for~\eqref{e:IPM} by studying a kernel-dependent joint distribution on the quadrature nodes called \emph{volume sampling}. 
In short, under generic assumptions on the kernel, Markov's inequality applied to \citep[Theorem 4]{BeBaCh20} shows that there is a constant $C>0$ such that under volume sampling, with probability $1-C\sigma_{n+1}/\epsilon$, 
\begin{equation}
    \label{e:ayoub_concentration}
    \sup_{f\in\mathcal{H}_K} \left\vert \int f\d\pi - \sum_{i=1}^n w_i f(x_i)\right\vert^2\leq \epsilon,
\end{equation}
where $\sigma_n$ is the $n$-th eigenvalue of the operator $f \mapsto \int K(., x)f(x)\,\mathrm{d}x$ on $L^2(\pi)$, and the weights $(w_i)$ are suitably chosen.
Because $\sigma_n$ can go to zero arbitrarily fast with $n$ (e.g., exponentially for the Gaussian kernel), \eqref{e:ayoub_concentration} attains a given confidence and error levels at smaller $n$ than under i.i.d. sampling. 
Downsides are that $(i)$ there is no exact algorithm yet for volume sampling that does not require computing the eigenvalues and eigenfunctions of the integral operator with kernel $K$, and $(ii)$ the dependence of $w_i$ on all nodes makes it hard to derive, e.g., a central limit theorem.

In the spirit of kernel thinning, \citet{chatalic2025efficient} also achieve optimal bounds e.g. in Sobolev RKHSs.
Yet, their procedure requires a large number of i.i.d inputs points, for instance exponential in $n$ for the Gaussian kernel, if the goal is to perform numerical integration with $n$ nodes.

\paragraph{Concentration for Gibbs measures}
We informally define a (Gibbs) measure on $\left(\mathbb{R}^{d}\right)^{n}$ by
\begin{equation}
    \label{e:gibbs_measure_and_energy}
    \frac{\d \PP}{\d X_n}(X_n) = \frac{\mathrm{e}^{-\beta_n H_n(X_n)}}{Z_{n, \beta_n}^{V}}, \quad H_n(X_n) = \frac{1}{2n^2}\sum_{i \neq j }K(x_i, x_j) + \frac{1}{n} \sum_{i = 1}^{n} V(x_i),
\end{equation} 
where $X_n$ is short for $(x_1, \dots, x_n)$, $\beta_n > 0$ is called \emph{inverse temperature}, $V:\mathbb{R}^d\rightarrow\mathbb{R}$ and $Z_{n, \beta_n}^{V}$ denotes the normalizing constant such that $\PP$ is a probability measure.
There are assumptions to be made on $K$ and $V$ to guarantee that \eqref{e:gibbs_measure_and_energy} defines a \emph{bona fide} probability distribution; see Section~\ref{s:results}.
$H_n$ in \eqref{e:gibbs_measure_and_energy} can be recognized to be a discrete analogous to $I_{K}$ in \eqref{e:IPM}. 
Intuitively, points distributed according to \eqref{e:gibbs_measure_and_energy} tend to correspond to low pairwise kernel values $K(x_i, x_j)$ (we say that they \emph{repel} by a force given by the kernel), yet stay confined in regions where $V$ is not too large. 
We also emphasize that the zero temperature limit -- informally taking $\beta_n = +\infty$ -- corresponds to finding the deterministic minimizers of $H_n$, which in turn intuitively correspond to minimizers of $I_K$; see \citep{serfaty2018systems} for precise results.
We focus in this paper on the so-called \emph{low temperature regime} $\beta_n / n \rightarrow + \infty$ (denoted in the sequel by $\beta_n \gg n$), 
in which one can hope to observe properties of the Gibbs measure \eqref{e:gibbs_measure_and_energy} that depart from those of i.i.d. sets of $n$ points.

Asymptotic properties of \eqref{e:gibbs_measure_and_energy} as $n\rightarrow\infty$ have been studied by \citep{chafai2014first} for a general $K$. 
As for non-asymptotic counterparts, concentration inequalities have been obtained for some singular kernels,\footnote{
    By \emph{singular}, we mean that $K(x,x)=+\infty$ for all $x\in\mathbb{R}^d$.
} known as the Coulomb and Riesz kernels \citep{chafai2018concentration,garcia2022generalized}.  
For instance, \citep{chafai2018concentration} prove for the Coulomb kernel that whenever $r > n^{-1/d}$,
\begin{equation}
    \label{e:concentration_coulomb}
    \PP(W_1(\mu_n, \mu_V) > r) \leq \exp\left(- c \beta_n r^2\right),
\end{equation}
where $\mu_V$ is the so-called \emph{equilibrium measure}, which depends in a non-trivial way on $V$ and $K$.
The concentration \eqref{e:concentration_coulomb} improves on the i.i.d. concentration in \eqref{e:concentration_iid}. 
Besides being valid for values of $r$ down to $n^{-1/d}$, the speed $n$ in the exponential is replaced by $\beta_n$, which can increase arbitrarily fast, at the price of replacing the target measure by the equilibrium measure of the system.\footnote{
    While a fast-growing $\beta_n$ implies better theoretical guarantees, the price of (approximately) sampling from $\PP$ intuitively increases with $\beta_n$, introducing a trade-off in practice; see Section~\ref{s:experiments}.
}
After choosing a suitable potential $V$ such that $\mu_V = \pi,$ we rephrase this bound as a uniform quadrature guarantee.
\begin{corollary}[\citet{chafai2018concentration}]\label{c:quadrature_coulomb}
Fix $\epsilon > 0$ and $\delta \in (0, 1)$.
Let $x_1, \dots, x_n \sim \mathbb{P}_{n, \beta_n}^{V}$ in \eqref{e:gibbs_measure_and_energy} and $\mu_n = n^{-1}\sum_{i = 1}^{n}\delta_{x_i}$.
For any $n$ such that $n^{-1/d} \leq \epsilon$ and $\exp\left(-c \beta_n \epsilon^{2}\right) \leq \delta$,
\begin{equation}
		\label{e:quadrature_concentration_coulomb}
		\mathbb{P}_{n, \beta_n}^{V} \left(\underset{\lVert f\rVert_{\mathrm{Lip}} \leq 1}{\sup}\left\lvert \int f\,\mathrm{d}(\mu_n - \pi)\right\rvert \leq \epsilon\right) \geq 1-\delta.
\end{equation}
\end{corollary}
As long as $\beta_n\gg n$, for a fixed confidence level $\delta$ and worst-case error $\epsilon$, as soon as $n$ is big enough so that $n^{-1/d} \leq \epsilon$, the constraints in Corollary~\ref{c:quadrature_coulomb} are achieved with a smaller $n$ than for i.i.d. samples in \eqref{e:quadrature_concentration_iid}.
Fewer quadrature nodes are required by the Gibbs measure to achieve the same guarantee.

Results like \citep{chafai2018concentration} are motivated by statistical physics and focus on a particular family of singular kernels. 
The price of singularity is quite long and technical proofs. 
On the other hand, in machine learning, we typically consider bounded kernels like the Gaussian or Matern kernel. 
Our main result is a version of \eqref{e:concentration_coulomb} that is valid for very general \emph{bounded} kernels, bringing an improvement over i.i.d. sampling similar to Corollary~\ref{c:quadrature_coulomb}, with $n^{-1/d}$ even replaced by $n^{-1/2}$.
Maybe surprisingly, while our proof follows the lines of \citep{chafai2018concentration}, we were able to considerably simplify the more technical arguments.
We hope our work helps transfer tools and concepts from the theory of Gibbs measures to the study of IPM-based quadrature. 

As a final note on existing work, and to prepare for the discussion in Section~\ref{s:discussion}, we remark that central limit theorems for Gibbs measures like \eqref{e:gibbs_measure_and_energy} (with speed depending on $\beta_n$) are very subtle mathematical results. \citep{leble2018fluctuations} and \citep{bauerschmidt2016two} have obtained a CLT only in dimension two so far, and yet only for the (singular) Coulomb kernel.
While important steps have been made towards larger dimensions for the Coulomb kernel \citep{serfaty2023gaussian}, this remains an important and difficult open problem in statistical physics. 
As a consequence, direct comparison with the $n^{-1/2}$ rate appearing in the CLTs of MCMC chains is currently out of reach.

\section{Main results}
\label{s:results}

We first introduce key notions to understand the limiting behavior of Gibbs measures, like the \emph{equilibrium measure}. 
We then introduce our Gibbs measure on quadrature nodes, and state our main result, which features the equilibrium measure. 
In the last paragraph, we explain how to choose the parameters of the Gibbs measure so that the equilibrium measure is a given target distribution $\pi$.

\subsection{Energies and the equilibrium measure}
\label{s:energies}
Let $d\geq 1$,
$    
    K\, : \mathbb{R}^{d} \times \mathbb{R}^{d} \rightarrow \mathbb{R}\cup \{\pm \infty\}
$ 
and 
$
    V \, : \mathbb{R}^{d}  \rightarrow \mathbb{R} \cup \{+\infty\}.
$
For reasons that shall become clear shortly, we call $K$ the \emph{interaction kernel}, and $V$ the \emph{external potential}.
Assumptions on $K$ and $V$ will be given to make the following definition meaningful.
We mention that the physics-inspired names for the different notions are useful to the intuition; see Section~\ref{app:physics_vocabulary}.

\begin{definition}[Energies]
    \label{d:energies}
    Whenever they are well-defined, we introduce the following quantities, for signed Borel measures $\mu, \nu$ on $\mathbb{R}^{d}$.
    The \emph{interaction potential}, or \emph{kernel embedding}, of $\mu$, is defined as 
    $
        U_{K}^{\mu}(z) = \int K(z, y)\,\mathrm{d}\mu(y)
    $, where $z\in \mathbb{R}^{d}$.
    The \emph{interaction energy} between $\mu$ and $\nu$ is defined as
    $
        I_K(\mu, \nu) = \iint K(x, y)\,\mathrm{d}\mu(x)\,\mathrm{d}\nu(y).
    $
    When $\mu = \nu$, we simply write $I_{K}(\mu)=I_K(\mu, \mu)$.
    Finally, we let
    $
        I_{K}^{V}(\mu) = \frac{1}{2}\iint \left\{K(x, y) + V(x) + V(y)\right\}\mathrm{d}\mu(x)\,\mathrm{d}\mu(y).
    $
\end{definition}

We will work under the following assumptions on $K$ and $V$.
The first one restricts our class of interaction kernels, insisting that points should repel, but that the interaction cannot be singular. 
\begin{assumption}
    \label{a:bounded_kernel}
    $K$ is symmetric, non-negative, continuous, and bounded on the diagonal: there exists some constant $C \geq 0$ such that $K(x, x) \leq C < \infty$ for all $x \in \mathbb{R}^{d}$.
\end{assumption}
Assumption~\ref{a:bounded_kernel} in particular ensures that $I_K(\mu)$ is well-defined for any probability measure, with possibly infinite value.
Our next assumption excludes pathological cases where $I_{K}$ does not induce a distance on probability distributions.
\begin{assumption}
    \label{a:isdp_kernel}
    $K$ is integrally strictly positive definite (ISDP), i.e. $I_{K}(\mu) > 0$ for any non-zero finite signed measure Borel $\mu$.
\end{assumption}
First, Assumptions~\ref{a:bounded_kernel} and \ref{a:isdp_kernel} allow most kernels used in machine learning \citep{RaWi06}, like the Gaussian or isotropic Matern kernels, as well as truncated singular kernels like the multiquadratic kernel
\citep{pronzato2020bayesian}.
Second, under Assumptions~\ref{a:bounded_kernel} and \ref{a:isdp_kernel}, $K$ is finite on the diagonal and ISDP, so that it is in particular positive definite.
We can then consider the RKHS $\mathcal{H}_{K}$ induced by the kernel $K$ \citep{BeTh11}.
An easy consequence of the Cauchy-Schwarz inequality in $\mathcal{H}_{K}$ is that for any $x, y \in \mathbb{R}^{d}$, $0 \leq K(x, y) \leq C$. 
In particular, $I_{K}(\mu, \nu)$ and $U_{K}^{\mu}(z)$ from Definition~\ref{d:energies} are well-defined and finite for all finite signed Borel measures; see \citep{pronzato2020bayesian} for more details.
We henceforth denote by $\mathcal{E}_{K}$ (respectively $\mathcal{E}_{K}^{V}$) the set of finite signed Borel measures with finite interaction energy $I_K(\mu)$ (respectively, with finite energy $I_{K}^{V}(\mu)$).
The following known duality formula then links energy minimization and quadrature guarantees for integrands in the unit ball of $\mathcal{H}_{K}$.

\begin{proposition}[\citet{gretton2010hilbert}]\label{p:dual}
Under Assumptions~\ref{a:bounded_kernel} and \ref{a:isdp_kernel}, 
for probabilities $\mu, \nu$ in $\mathcal{E}_{K}$, let
$
    \mathrm{MMD}_{K}(\mu, \nu) = \underset{\lVert f\rVert_{\mathcal{H}_{K}} \leq 1}{\sup}\left\lvert \int f\,\mathrm{d}(\mu - \nu)\right\rvert.
$
Then $\mathrm{MMD}_{K}(\mu, \nu) = \left(I_{K}(\mu - \nu)\right)^{1/2}$.
\end{proposition}

We add an assumption to make sure  that $V$ is strong enough a confining term. Together with Assumption~\ref{a:bounded_kernel}, this makes $I_{K}^{V}$ well-defined for any probability measure,\footnote{
    Unlike $I_K$, we shall only evaluate $I_K^V$ on probability measures.
}
with possibly infinite value. 
\begin{assumption}
    \label{a:confinement}
    $V$ is lower semi-continuous, finite everywhere and  $V(x) \longrightarrow + \infty$ when $\lvert x \rvert \longrightarrow +\infty.$ Moreover,
    there exists a constant $c > 0$ such that $\int \exp\left(-c V(x)\right)\,\mathrm{d}x < \infty$.
\end{assumption}

We are now ready to consider the minimizers of $I_{K}^{V}$.

\begin{proposition}\label{p:eq}
Let $K$ satisfy Assumptions~\ref{a:bounded_kernel} and \ref{a:isdp_kernel}, and $V$ satisfy Assumption~\ref{a:confinement}. 
Then $I_{K}^{V}$ is lower semi-continuous, has compact level sets and 
    $
         I_{K}^{V}(\mu) > - \infty
    $
    for any probability distribution $\mu$ on $\mathbb{R}^d$. Moreover, if $\mu \in \mathcal{E}_{K}^{V}$, then $I_{K}(\mu)$ and $\int \lvert V\rvert\,\mathrm{d}\mu$ are finite, and 
    $
        I_{K}^{V}(\mu) = \frac{1}{2}I_{K}(\mu) + \int V\,\mathrm{d}\mu
    $. $I_{K}^{V}$ is strictly convex on the convex non-empty set $\mathcal{E}_{K}^{V}$, $I_{K}^{V}$ has a unique minimizer $\mu_V$ over the set of probability measures on $\mathbb{R}^{d}$, called the \emph{equilibrium measure}, and the support of $\mu_V$ is compact.
    \end{proposition}
The proof borrows from \citep{chafai2014first} and \citep{pronzato2020bayesian}; see Section~\ref{s:proofs}.

\subsection{Concentration for the Gibbs measure}
\label{s:concentration_result}
We saw in Section~\ref{s:introduction} that herding-like algorithms rely on finding a configuration of points $\{x_1, \dots, x_n\}$ that minimizes the interaction energy $I_{K}(\frac{1}{n}\sum_{i = 1}^{n} \delta_{x_i} - \pi)$.
In this paper, we rather consider points drawn from a distribution that favors small values for an empirical proxy of this interaction energy.

\begin{definition}[Gibbs measure]
    Let $K$ satisfy Assumptions~\ref{a:bounded_kernel} and \ref{a:isdp_kernel}, and $V$ satisfy Assumption~\ref{a:confinement}. 
    Let $\beta_n \geq 2 c n$, where $c$ is the constant of Assumption~\ref{a:confinement}.
    The measure $\mathbb{P}_{n, \beta_n}^{V}$ defined on $\mathbb{R}^{dn}$ by \eqref{e:gibbs_measure_and_energy} is a \emph{bona fide} probability measure. 
    In particular, $Z_{n, \beta_n}^{V} \in (0,+\infty)$ \citep{chafai2014first, chafai2018concentration}.
\end{definition} 

When $x_1, \dots, x_n \sim \mathrm{d}\mathbb{P}_{n, \beta_n}^{V}$, we will henceforth denote by $\mu_n = \frac{1}{n}\sum_{i = 1}^{n} \delta_{x_i}$ the associated empirical measure.
We saw in Section~\ref{s:related_work} that $\mu_n$ converges to the equilibrium measure $\mu_V$, with a large deviations principle at speed $\beta_n$.
We are able to give non-asymptotic guarantees on this convergence via a concentration inequality, which is the main result of the paper.

\begin{theorem}[Concentration inequality]
    \label{th:concentration_inequality}
    Let $K$ satisfy Assumptions~\ref{a:bounded_kernel} and~\ref{a:isdp_kernel}, and $V$ satisfy Assumption~\ref{a:confinement}. 
    Further assume that the associated equilibrium measure $\mu_V$ has finite entropy
    \begin{align}
      S(\mu_V) = -\int \log\mu_V\,\mathrm{d}\mu_V. 
    \end{align}
    Let $\beta_n \geq 2 c n$, where $c$ is the constant of Assumption~\ref{a:confinement}.
    Then, for any $n \geq 2$ and for any $r > 0$,
    \begin{align}
        \mathbb{P}_{n, \beta_n}^{V} & \left(I_{K}\left(\mu_n - \mu_V\right) > r^2\right) 
        \leq \exp\left(-\frac{1}{4} \beta_n r^2 + n(c_1 + \frac{\beta_n}{n^2} c_2)\right),
        \label{e:main_result_c4}
    \end{align}
    with $c_1, c_2$ given by
    \begin{align}
      &c_1 = c I_{K}^{V}(\mu_V)+\log\int_{\mathbb{R}^{d}}\exp\left(-c V(x)\right)\,\mathrm{d}x - S(\mu_V),\\
      &c_2 = \frac{1}{2}C - \frac{1}{2}I_{K}(\mu_V),
    \end{align}
    and $C$ is the constant of Assumption~\ref{a:bounded_kernel}.
\end{theorem}

    We emphasize again that by Proposition~\ref{p:dual}, Equation \eqref{e:main_result_c4} provides a non-asymptotic confidence interval for the worst-case quadrature error in the unit ball of the RKHS $\mathcal{H}_{K}$.
    Note that the bound is only interesting in the regime $\beta_n /n \rightarrow +\infty$, where the temperature $1/\beta_n$ goes down quickly enough. 
    A classical choice of temperature scale is $\beta_n = \beta n^2$ where $\beta > 0$.
    We rephrase our main result to put forward the sub-Gaussian decay in the bound.

\begin{corollary}
    \label{c:sub_Gaussian_decay}
    Under the assumptions of Theorem~\ref{th:concentration_inequality}, let further $\beta_n \gg n$.
    Then, for any $n \geq 2$ and for any 
    \begin{align}
        \label{e:condition_r_subgauss}
        r \geq 4\max \left(\sqrt{c_2}n^{-1/2}, \sqrt{c_1}\left(\beta_n / n\right)^{-1/2}\right),
    \end{align}
    \begin{equation}
        \label{e:simplified_main_result_c4}
        \mathbb{P}_{n, \beta_n}^{V}\left(\mathrm{MMD}_K\left(\mu_n, \mu_V\right) > r\right) \leq \exp\left(-\frac{1}{8}\beta_n r^2\right).
    \end{equation}
\end{corollary}
The proof of Corollary~\ref{c:sub_Gaussian_decay} is straightforward from Theorem~\ref{th:concentration_inequality}, itself proved in Section~\ref{s:proofs}.

In Corollary~\ref{c:sub_Gaussian_decay}, when $\beta_n \geq v n^2$ for some constant $v > 0$, the condition on $r$ becomes $r > u_0 n^{-1/2}$.
    This makes it clear that our bound on the MMD still decreases as $n^{-1/2}$, like classical Monte Carlo, but that the exponent in the coverage \eqref{e:simplified_main_result_c4} has been improved from $\Omega(n^{-1})$ into $\Omega(\beta_{n}^{-1})$.
    Indeed, under i.i.d or MCMC sampling, results like~\eqref{e:concentration_iid} feature $n$ instead of $\beta_n$ in the right-hand side of~\eqref{e:simplified_main_result_c4}, and $\beta_n/ n\rightarrow +\infty$. 
    This means that for a fixed precision level $r > 0$ for the MMD, fewer points will be needed with a Gibbs measure to achieve the same coverage as i.i.d or MCMC points, as soon as the condition~\eqref{e:condition_r_subgauss} on $r$ still holds.
    This fast-increasing coverage probability of our confidence interval is a trace of the repulsion in the Gibbs measure.

    As another comment, integrating~\eqref{e:simplified_main_result_c4} over all $r > 0$ yields
    \begin{align}
    \mathbb{E}\left[\mathrm{MMD}_K\left(\mu_n, \mu_V\right)\right] \leq \frac{u_1}{\sqrt{n}} + \frac{u_2}{\sqrt{\beta_n}}
    \end{align}
    for some constants $u_1, u_2 > 0$, and we thus again recover the usual ``slow'' $n^{-1/2}$ rate for the MMD.
    Again, the improvement from our result compared to MCMC or i.i.d points rather comes at the level of coverage.

We now explain in concrete terms how Corollary~\ref{c:sub_Gaussian_decay} implies that Monte Carlo integration with our Gibbs measure and with respect to an arbitrary target distribution $\pi$ outperforms crude Monte Carlo.

\subsection{Application to guarantees for probabilistic herding}

Let $d \geq 1$ and $\pi$ be a probability measure on $\mathbb{R}^{d}$, which we assume to be our target. 

\begin{assumption}
    \label{a:compact_support}
    The support $S_\pi\subset \mathbb{R}^d$ of $\pi$ is compact, and $\pi$ has finite entropy, in the sense that $\pi$ has a density $\pi'$ w.r.t. Lebesgue, and that $ - \int \log \pi'(x) \mathrm{d}\pi(x) <\infty.$
\end{assumption}

The following proposition shows that, for a given kernel $K,$ we can choose $V$ so that $\mu_V = \pi$, assuming prior computation of the kernel embedding $U_{K}^{\pi}(z)$ for all $z \in \mathbb{R}^{d}$.

\begin{proposition}\label{p:inv}
    Let $K$ satisfy Assumptions~\ref{a:bounded_kernel} and \ref{a:isdp_kernel}, and $\pi$ satisfy Assumption~\ref{a:compact_support}.
    In particular, there exists $R>0$ such that $S_{\pi} \subset B(0, R)$, where $B(0, R)$ is the closed Euclidean ball. 
    Let $\Phi : \mathbb{R}^d \rightarrow \mathbb{R}_+$ be any continuous function such that
    $\Phi = 0$ on $\partial B(0, R),$ $\Phi(z) \rightarrow +\infty$ as $\lvert z \rvert \rightarrow +\infty$ and
    $
        \int_{\{\lvert x \rvert > R\}} e^{-\Phi(x)}\,\mathrm{d}x < \infty.
    $   
    Then, setting $V^{\pi}(z) = - U_{K}^{\pi}(z)$ when $z \in B(0, R)$ and $V^{\pi}(z) = -U_{K}^{\pi}(z) + \Phi(z)$ otherwise, $V^{\pi}$ satisfies Assumption~\ref{a:confinement} with $c = 1$ and $\mu_{V^{\pi}} = \pi$.
\end{proposition}
This is a standard result, which relies on the so-called Euler-Lagrange characterization of the equilibrium measure. 
We give a proof in Section~\ref{s:proofs}, which is inspired by Corollary $1.4$ of \citep{chafai2014first}, who treat the more difficult case of singular interactions.
A classical choice of $\Phi$ is $\Phi(z) = \lvert z\rvert^2 - R^2$.
In machine learning terms, Proposition~\ref{p:inv} says that a suitably penalized kernel embedding is a good choice of confining potential. 
Of course, this choice of $V$ requires the ability to evaluate the kernel embedding $U_{K}^{\pi}$, and we fall back here onto a standard limitation in the herding literature \citep{chen2010super, bach2012herding}.
While closed form expression for $U_{K}^{\pi}$ are available \citep[Table 1]{briol2019probabilistic} for simple pairs of kernel and target $(K, \pi)$, there is no general expression. We will rely on approximations of this quantity in the numerical experiments of Section~\ref{s:experiments}.

With $\pi$ now our equilibrium measure, Corollary~\ref{c:sub_Gaussian_decay} implies a uniform quadrature guarantee.
\begin{corollary}\label{c:quadrature_bounded_kernel}
    Let $K$ satisfy Assumptions~\ref{a:bounded_kernel} and \ref{a:isdp_kernel}, and $\pi$ satisfy Assumption~\ref{a:compact_support}.
    Set $V = V^{\pi}$ as in Proposition~\ref{p:inv}, assume that $\beta_n \gg n$, and let $x_1, \dots, x_n \sim \PP$.
    Let further $\epsilon > 0$ and $\delta \in (0, 1)$.
    For any $n$ such that       
    \begin{align}
          \epsilon \geq 4 \max\left(\sqrt{c_2}n^{-1/2}, \sqrt{c_1}\left(\beta_n / n\right)^{-1/2}\right)
      \end{align}
      and $\exp\left(-\beta_n \epsilon^{2}/8\right) \leq \delta$,
    with probability larger than $1 - \delta$,
\begin{equation}
		\label{e:quadrature_concentration_bounded_kernel}
		\underset{\lVert f\rVert_{\mathcal{H}_{K}} \leq 1}{\sup}\left\lvert \int f\,\mathrm{d}(\mu_n - \pi)\right\rvert \leq \epsilon.
\end{equation}
\end{corollary}
Corollary~\ref{c:quadrature_bounded_kernel} follows from Corollary~\ref{c:sub_Gaussian_decay} and Proposition~\ref{p:inv}.
Compared to Theorem~\ref{th:bolley}, for a fixed worst-case integration error $\epsilon$ and confidence level $\delta$, fewer points are required under $\PP$ than under i.i.d samples from $\pi$.

Assumptions on the RKHS are generic, and one can see that our concentration result does not depend on the kernel $K$ or the target $\pi$ except for the constants $c_1$ and $c_2$.
This is to be expected from the statistical physics literature, since Gibbs measures typically exhibit \emph{universality}: the fluctuations or local properties of the point process are expected to be independent of its global parameters.
Another strength of our result is that the constants and speeds are quite explicit.
We could be worried that the constants $c_1$ and $c_2$ hide exponential dependencies in the dimension: we show in the following proposition that this is not the case.

\begin{proposition}[Behavior of the constants]
  \label{l:bounds_constants_concentration}
  We work under the assumptions of Corollary~\ref{c:quadrature_bounded_kernel}.
  Recall that $C$ is a positive constant such that $0 \leq K(x, x) \leq C$ for all $x \in \mathbb{R}^{d}$, and that $R$ is a positive constant such that the support of $\pi$ is included in $B(0, R)$.
  We further choose $\Phi(x) = \lVert x\rVert^2 - R^2$ in Proposition~\ref{p:inv}.
  Then
  \begin{align}
    &c_1 \leq 2C-S(\pi) + d\left(\frac{\sqrt{2}+1}{2}\log(\pi)+\log R\right) + R^2 -\log \Gamma\left(d/2 +1\right),\\
    &c_2 \leq \frac{1}{2}C,
  \end{align}
  where $\Gamma$ is Euler's Gamma function.
  In particular,
  \begin{align}
    \log \Gamma(d/2+1) = \mathcal{O}\left(d\log d\right)
  \end{align}
  as $d\rightarrow + \infty$, using Stirling's formula.
  Moreover,
  \begin{align}
    -S(\pi) := \int \log \pi'(x) \,\mathrm{d}\pi(x) \leq \lVert \log \pi'\rVert_{\infty},
  \end{align}
  as soon as the density $\pi'$ of $\pi$ is continuous on its compact support $S_{\pi}$.
\end{proposition}
The proof is given in Section~\ref{s:proofs}.
We thus see from Proposition~\ref{l:bounds_constants_concentration} that there is no intrinsic bad dependence of the constants in the dimension $d$.
Of course, one could choose a kernel $K$ that is very close to a singular kernel, so that the constant $C$ blows up exponentially with $d$, but this issue would transfer to all kernel-based methods for such a $K$.
The same goes with choosing $\pi$ such that $R$ scales badly with the dimension $d$.
Moreover, the bounds in Proposition~\ref{l:bounds_constants_concentration} are far from sharp, and we completely discarded the negative energy terms in~\eqref{e:constants_concentration}.

Another high-level remark is that the concentration bounds in Corollary~\ref{c:quadrature_bounded_kernel} suggest that considering $\beta_n$ as large as possible would provide the best guarantee in terms of coverage.
This is to be tempered with the fact that sampling points from the Gibbs measure $\mathbb{P}_{n, \beta_n}^{V^{\pi}}$ in \eqref{e:gibbs_measure_and_energy} is intuitively harder when $\beta_n$ is very large.

    Relatedly, the elephant in the room is that there is no general algorithm to draw points $x_1, \dots, x_n$ from $\mathbb{P}_{n, \beta_n}^{V^{\pi}}$, and we rely on MCMC to do so in this paper, see Section~\ref{s:experiments}.
    When $\beta_n$ is very large, the Gibbs distribution $\mathbb{P}_{n, \beta_n}^{V^{\pi}}$ becomes very spiked around the minimizers of the energy~\eqref{e:gibbs_measure_and_energy}: this is typically a difficult setting for MCMC algorithms, meaning that many iterations might be needed to get a good approximation in total variation distance to the target.
    Quantifying how this approximated sampling interplays with the statistical guarantees of Corollary~\ref{c:quadrature_bounded_kernel} is an avenue for future work.
Meanwhile, the following corollary gives a general, yet unpractical, result for the statistical guarantees obtained when combining our statistical guarantees with an approximate MCMC sampler.
\begin{corollary}
  \label{cor:mixing_time_gibbs}
  We work under the assumptions of Corollary~\ref{c:quadrature_bounded_kernel}.
  Consider a Markov kernel $P_n$ on $\mathbb{R}^{dn}$ and an initial distribution $\xi_n$, and define the mixing time
  \begin{align}
      \label{e:mixing_time_tv_gibbs}
      \mathcal{T}_{\mathrm{TV}, n, \beta_n}(\xi_n, \delta) = \inf\{ T \geq 1\,:\, \lVert \xi_n P_{n}^{T} - \mathbb{P}_{n, \beta_n}^{V^{\pi}}\rVert_{\mathrm{TV}} < \delta\}.
  \end{align}
  We denote by $x_{1}^{(T)}, \dots, x_{n}^{(T)}$ the state of the Markov chain with kernel $P_n$ started at $\xi_n$ after $T$ iterations, and by $\mu_{n}^{(T)}$ the associated empirical measure.
  Let $\epsilon > 0$ and $\delta \in (0, 1)$.
  Let $n$ be such that 
      \begin{align}
          \epsilon \geq 4 \max\left(\sqrt{c_2}n^{-1/2}, \sqrt{c_1}\left(\beta_n / n\right)^{-1/2}\right)
      \end{align}
  and 
  \begin{align}
  2\exp\left(-\beta_n \epsilon^{2}/8\right) \leq \delta.
  \end{align}
  Further assume that
  $
    T \geq \mathcal{T}_{\mathrm{TV}, n, \beta_n}(\xi_n, \delta/2).
  $
  Then, with probability larger than $1-\delta$,
  \begin{equation}
		\label{e:quadrature_concentration_bounded_kernel_mixing}
		\underset{\lVert f\rVert_{\mathcal{H}_{K}} \leq 1}{\sup}\left\lvert \int f\,\mathrm{d}\left(\mu_{n}^{(T)} - \pi\right)\right\rvert \leq \epsilon.
\end{equation}
\end{corollary}
The proof just comes from the inequality
\begin{align}
  \mathbb{P}\left(\mathrm{MMD}_{K}\left(\mu_{n}^{(T)}, \pi\right) > \epsilon\right) \leq \mathbb{P}_{n, \beta_n}^{V^{\pi}}\left(\mathrm{MMD}_{K}\left(\mu_{n}, \pi\right) > \epsilon\right)+ \lVert P_{n}^{T} \xi - \mathbb{P}_{n, \beta_n}^{V^{\pi}}\rVert_{\mathrm{TV}},
\end{align}
where $\mathrm{TV}$ denotes the total variation distance,
and we applied Corollary~\ref{c:quadrature_bounded_kernel} with $\delta/2$.

    Corollary~\ref{cor:mixing_time_gibbs} gives tighter coverage than with $n$ i.i.d points for the algorithm we actually use in the experimental Section~\ref{s:experiments}, provided the mixing time $\mathcal{T}_{\mathrm{TV}, n, \beta_n}(\xi_n, \delta)$ is finite and scales reasonably with the dimensions $d, n$ of the problem.
    Yet, to our knowledge, there are no usable bounds for such a mixing time in the MCMC literature, in particular because the target distribution is not log-concave in general.
    This is the main limitation of our approach, which we will discuss in Section~\ref{s:discussion}.

    For now, we will use standard Markov kernels $P_n$ in Section~\ref{s:experiments}, along with some heuristics to take into account the dependence in $\beta_n$.
    The fact that we will observe tighter confidence intervals than for classical Monte Carlo, as suggested by our concentration bound, is a token that approximate sampling from our Gibbs measure is a reasonable goal.
    We now turn to a practical description of the sampling algorithm, discuss the main bottlenecks, and illustrate our results on toy examples.


\section{Experiments}
\label{s:experiments}


\subsection{Approximately sampling from $\PP$}
\label{s:sampling}
\paragraph{Markov kernel}
There is no known algorithm to sample exactly from \eqref{e:gibbs_measure_and_energy} for generic $K$ and $V$, so we resort to MCMC, namely the Metropolis-adjusted Langevin algorithm (MALA; \citealp{robert2004monte}). 
To wit, MALA is Metropolis--Hastings (MH) with proposal 
\begin{equation}
    \label{e:mala_proposal}
    y | y_{t} \sim \mathcal{N}\left(y_{t} - \alpha \beta_n \nabla H_n(y_{t}), 2 \alpha I_{d n}\right),
\end{equation}
where $y \in(\mathbb{R}^d)^n$, and $\alpha$ is a user-tuned step size parameter.
To sample from $\PP$, we run $T$ MALA iterations and keep \emph{only} the $T-$th iteration as our approximate sample.
An approximate sample from our Gibbs measure thus costs $\mathcal{O}(Tn^2)$, since we need to compute pairwise interactions at each MALA iteration.
We refer to Section~\ref{a:comput:complexity} for more on the computational complexity and its comparison to standard MCMC quadrature and herding algorithms.
At this stage, we simply recall that in the herding setting of expensive-to-evaluate integrands, a larger compute time is justified as long as it comes with a more efficient use of the $n$ integrand evaluations.
We now discuss how we choose the step-size $\alpha$ and give heuristics on mixing times for $\mathbb{P}_{n, \beta_n}^{V^{\pi}}$.

\paragraph{Scaling the step-size}
We choose the MALA step size as $\alpha = \alpha_0\beta_{n}^{-1}$, and $\alpha_0$ is manually tuned at the beginning of each run so that acceptance reaches $50\%$. 
Having $\alpha$ decrease at least as $\beta_n^{-1}$ intuitively avoids the distance between two consecutive MALA states to grow with $n$; see the MALA proposal \eqref{e:mala_proposal}.
This choice is motivated by the fact that we do not wish for the number of iterations $T$ in Corollary~\ref{cor:mixing_time_gibbs} to scale with $\beta_n$ to reach a good approximation of $\PP$ and thus $\pi$.

In fact, consider a target measure $\mu = e^{-h(x)}\,\mathrm{d}x$ on $\mathbb{R}^{d'}$ which is $L$-smooth and $m$-strongly log-concave in the sense that $h$ is $\mathcal{C}^2$ and
\begin{align}
  \label{e:strongly_log_concave_assumption}
  m I_{d'} \prec \nabla^2 h \prec L I_{d'}
\end{align}
with $L > m > 0$.
Recall that for any initial distribution $\xi$ and any $\delta > 0$, we denote by
\begin{align}
  \mathcal{T}_{\mathrm{TV}}(\xi, \delta) = \inf\{T \geq 1: \lVert \xi P^T  - \pi\rVert_{\mathrm{TV}} < \delta\}
\end{align}
the mixing time of the Markov chain with Markov kernel $P$, in total variation.
In particular, we write $\mathcal{T}_{\mathrm{TV}}^{\mathrm{MALA}}(\xi, \delta)$ the mixing time for MALA.

Assuming that the initial distribution $\xi$ of the Markov chain is already close to $\mu$ --- this is called a \emph{warm start} --- \citep[Theorem 1]{dwivedi2019mixing} show that
\begin{align}
  \label{e:mixing_time_mala}
  \mathcal{T}_{\mathrm{TV}}^{\mathrm{MALA}}(\xi, \delta) = \mathcal{O}\left(\max\{d' \kappa, \sqrt{d'}\kappa^{3/2}\}\log\left(\frac{1}{\delta}\right)\right),
\end{align}
where $\kappa = L/m$, as soon as
\begin{align}
  \label{e:scaling_step_size_mixing_time}
  \alpha = \min\left\{ \frac{\sqrt{m}}{c(\delta, d') L \sqrt{d' L}}, \frac{1}{L d'}\right\},
\end{align}
where $c(\delta, d')$ is a function of $\delta$ and $d'$.
The polynomial dependence in the dimension $d'$ support the claim that MCMC algorithms are able scale to moderate to high dimensions with a good initialization and in the convex setting.

In our case, the target is $\mathbb{P}_{n, \beta_n}^{V^{\pi}}$.
We thus have $d' = d n$ and $m$ and $L$ would depend linearly on $\beta_n$.
If the energy $H_n$ in~\eqref{e:gibbs_measure_and_energy} was to satisfy~\eqref{e:strongly_log_concave_assumption} with $m_n$ and $L_n$, then so would the log-density of $\mathbb{P}_{n, \beta_n}^{V^{\pi}}$ with $\tilde{m}_n = \beta_n m_n$ and $\tilde{L}_n = \beta_n L_n$.
Thus, $\kappa = \tilde{L}_n/\tilde{m}_n = L_n / m_n$ would be independent of the choice of $\beta_n$ and so would the mixing time~\eqref{e:mixing_time_mala}.

From~\eqref{e:scaling_step_size_mixing_time}, this requires to consider a step-size 
\begin{align}
  \label{e:scaling_step_size_gibbs}
  \alpha = \beta_{n}^{-1} \min\left\{ \frac{\sqrt{m_n}}{c(\delta, d n) L_n \sqrt{d n L_n}}, \frac{1}{n L_n d}\right\},
\end{align}
thus depending on $\beta_n$ like $\beta_{n}^{-1}$.

This is only heuristic for numerous reasons.
First, $\mathbb{P}_{n, \beta_n}^{V^{\pi}}$ is not at all log-concave in general, let alone strongly log-concave.
Even if it was only weakly log-concave in the sense that~\eqref{e:strongly_log_concave_assumption} is satisfied with $m = 0$, which is still an unreasonable assumption for $\mathbb{P}_{n, \beta_n}^{V^{\pi}}$, then \citet{dwivedi2019mixing} show that the mixing time is $\mathcal{O}(d'^{2} L^{3/2}/\delta^{3/2})$.
Thus, there would this time be a dependence on $\beta_n$ coming from $L = \tilde{L}_n = \beta_n L_n$.
In Section~\ref{s:discussion}, we will discuss recent works trying to depart from the log-concave setting, yet the mentioned results are still far from ready-to-use in our case.

More importantly, with a small step-size $\alpha$ or large $\beta_n$, many MALA iterations might be necessary to reach a warm start, i.e a region of large mass under $\mathbb{P}_{n, \beta_n}^{V^{\pi}}$ (or any other target measure).
In practice, what we could recommend is to first run a fast-scaling deterministic method such as SVGD \citep{liu2016stein} or KSD descent \citep{korba2021kernel}, which are in a sense the warmest starts available in the literature, and use their result to initialize our algorithm, in order to benefit from our concentration guarantees.
In the end, these mixing properties are more a feature from MALA than from our Gibbs measure, and there likely are dedicated sampling algorithms that will outperform MALA at approximating $\PP$.

Nonetheless, there is thus a trade-off between large $\beta_n$ (i.e. concentrated $\mathbb{P}_{n, \beta_n}^{V^{\pi}}$ and small confidence regions) and sampling cost.
We advocate the use of $\beta_n \propto n^2$, which we could satisfyingly implement and which already provably outperforms vanilla MCMC in our theoretical result and in the toy experiments of this chapter.
In practice, we have been able to experiment up to $\beta_n \propto n^5$ but not higher, as numerical errors then pop up.

\paragraph{Approximation of the kernel embedding}
As usual in kernel herding, the kernel embedding $U_{K}^{\pi}$ is approximated by averaging along an MCMC chain of length $M\gg 1$, so that the computational cost of a sample from our Gibbs measure becomes $\mathcal{O}(Tn^2 + TMn)$; see Section~\ref{a:comput:kernel_approx} for more comments on this.

\paragraph{Toy examples}

\begin{figure}[!ht]
    \centering
    \subfloat[$\beta_n=n^{3/2}$ \label{f:three_samples:power_three_half}]{%
        \includegraphics[width=\threefig]{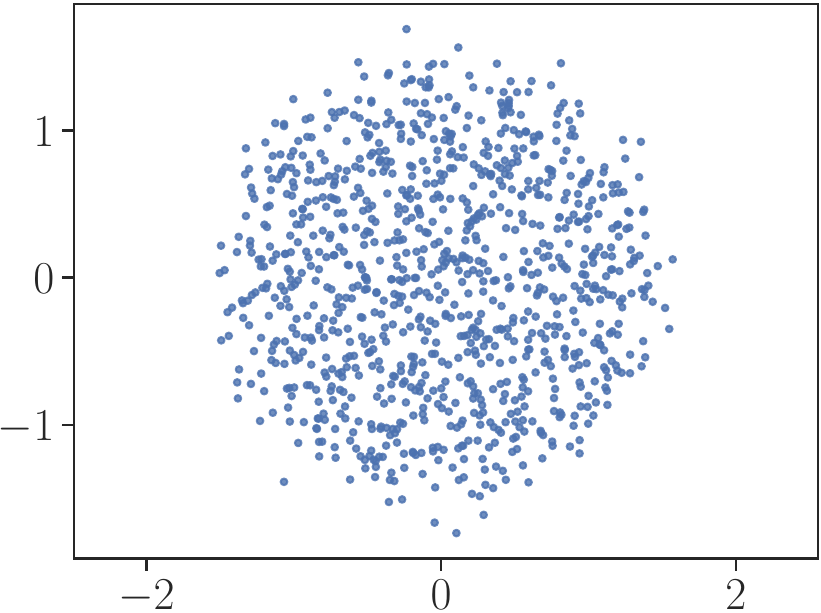}
    }
    \subfloat[$\beta_n = n^2$ \label{f:three_samples:power_two}]{
        \includegraphics[width=\threefig]{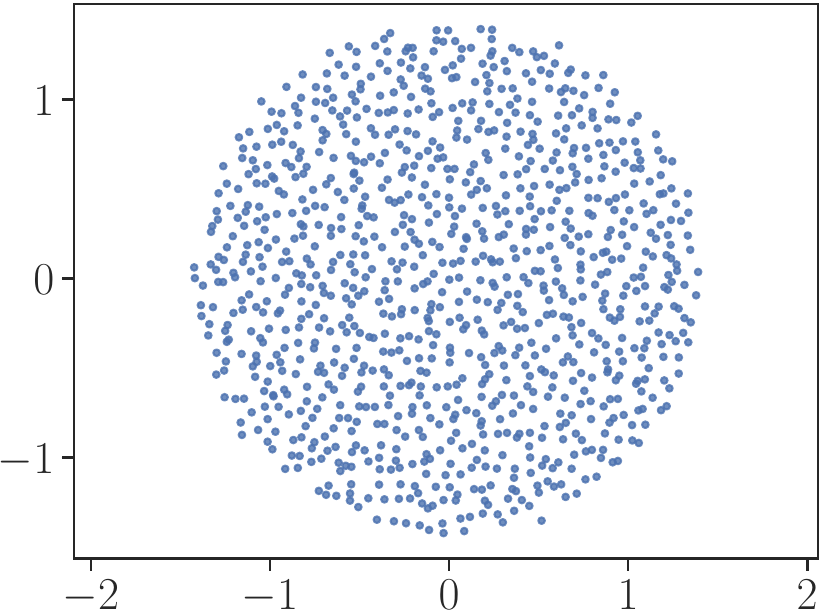}
    }
    \subfloat[$\beta_n = n^3$ \label{f:three_samples:power_three}]{
        \includegraphics[width=\threefig]{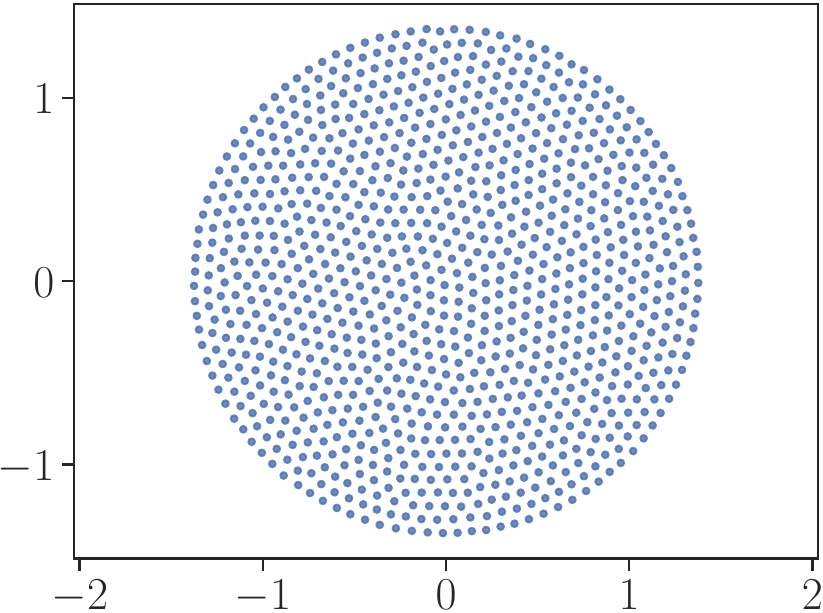}
    }
    \caption{Three independent approximate samples of $\PP$, with different temperature schedules.}
    \label{f:three_samples}
\end{figure}

We show in Figure~\ref{f:three_samples} how points sampled from $\PP$ look like in dimension $d=2$, for the quadratic potential $V:x\mapsto \lvert x\rvert^{2}/2$, and for different scalings of the inverse temperature $\beta_n$.
We consider the truncated logarithmic kernel $K(x, y) = - \log\left(\lvert x - y\rvert^{2} + \epsilon^{2}\right)$, where $\epsilon$ is a truncation parameter that we set to $\epsilon = 10^{-2}$, and set the number of particles $n$ to $1000$.
For $\epsilon=0$, the equilibrium measure is known to be uniform on the unit disk, and we expect it to be close to uniform in the truncated case as well.
In each panel of Figure~\ref{f:three_samples}, we show the state of the MALA chain after $T = 5000$ iterations, using Python and Jax \citep{jax2018github}.

We observe that the three empirical measures indeed approximate the uniform distribution on the disk, with more regular spacings as the inverse temperature grows.
Figure~\ref{f:three_samples:power_two} already shows a more regular arrangement of the points than under i.i.d. draws from the uniform distribution, while the lattice-like structure of Figure~\ref{f:three_samples:power_three} is a manifestation of what physicists call \emph{crystallization} \citep{serfaty2023gaussian}: the Gibbs measure is concentrated around minimizers of the energy.

Like most MCMC samplers, the performance of MALA is influenced by the absence of a warm start and the multimodality of the target. 
In Section~\ref{a:comput:multimodal}, we run an experiment targeting a mixture of 2D Gaussians, and illustrate how classical MCMC techniques like annealing tackle these issues.

In the next sections, we focus on illustrating the guarantees obtained in Section~\ref{s:results} on toy examples.
\subsection{Coverage of Monte Carlo integration}
\label{sub:coverage_c4}
    We first illustrate our main result from Corollary~\ref{cor:mixing_time_gibbs} stating that at a fixed precision level $\epsilon > 0$, Gibbs measures provide tighter confidence regions than vanilla Metropolis--Hastings targeting $\pi$ with the same number of points.
    \begin{figure}
        \centering
        \subfloat[$\mathbb{P}_{n, \beta_n}^{V^{\pi}}$ versus MCMC\label{f:coverage_gibbs_vs_mcmc_c4}]{%
            \includegraphics[width=\twofig]{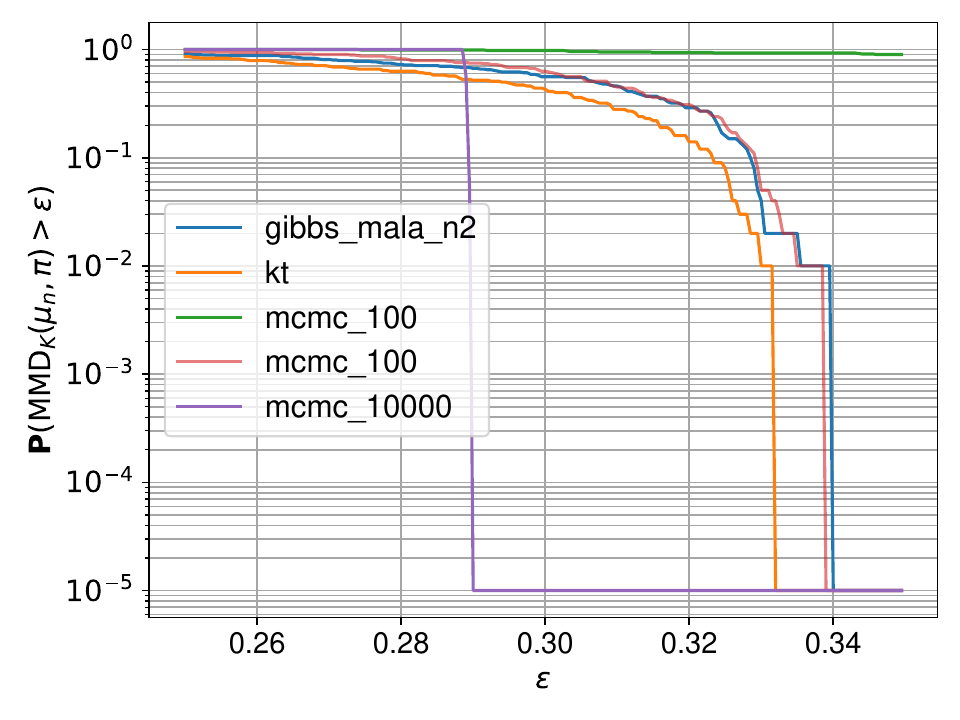}
        }
        \subfloat[MH and MALA kernels for sampling $\mathbb{P}_{n, \beta_n}^{V^{\pi}}$\label{f:coverage_mh_vs_mala_c4}]{
            \includegraphics[width=\twofig]{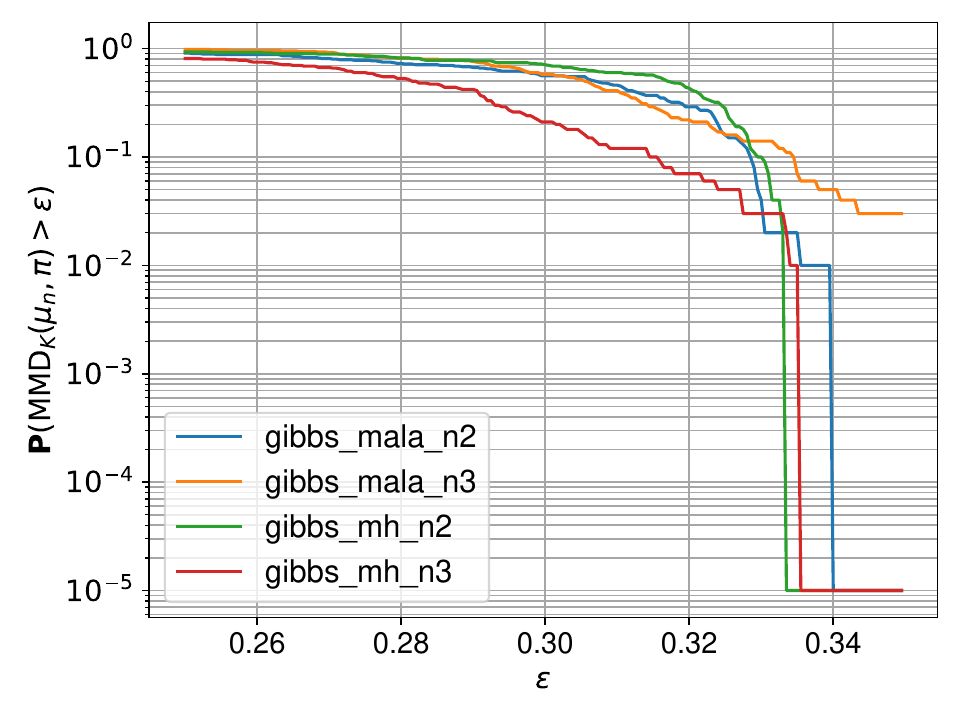}
        }
        \caption{Coverage comparisons between $\mathbb{P}_{n, \beta_n}^{V^{\pi}}$, MCMC and kernel thinning (\ref{f:coverage_gibbs_vs_mcmc_c4}) and for different sampling schemes for $\mathbb{P}_{n, \beta_n}^{V^{\pi}}$ (\ref{f:coverage_mh_vs_mala_c4}).}
        \label{f:figs_coverage_c4}
    \end{figure}
    We consider $\pi$ to be a Gaussian distribution in dimension $d = 10$, with variance $\sigma^2 = 1$, which we multiply by the indicator of $B(0, 2\sigma)$ to make sure $\pi$ has compact support.
    We further consider a Gaussian kernel $K(x, y) = \exp(-\lvert x - y\rvert^2 /2\sigma^2)$ with $\sigma = 1$.
    We want to empirically evaluate the coverage obtained by each integration method.
    To do so, for each integration method, we obtain $100$ independent samples $(x_1, \dots, x_n)$.
    We then compute the proportion of samples such that
    \begin{align}
    \mathrm{MMD_K}(\mu_n, \pi) > \epsilon,
    \end{align}
    for a grid of values of $\epsilon > 0$, where $\mu_n$ is the empirical measure $\mu_n = \frac{1}{n}\sum_{i = 1}^{n}\delta_{x_i}$.
    The MMD is approximated by replacing the integrals with respect to $\pi$ in
    \begin{align}
        \label{e:mmd_decomposition}
        I_{K}(\mu - \pi) = I_{K}(\mu) - 2 I_{K}(\mu, \pi) + I_{K}(\pi),
    \end{align}
    by averages over MCMC chains targeting $\pi$ with $T = 10^6$ iterations and independent of all the rest.
    This gives an estimate of the coverage
    \begin{align}
    \varepsilon\mapsto \mathbb{P}\left(\mathrm{MMD_K}(\mu_n, \pi) > \epsilon\right)
    \end{align}
    for each quadrature rule. 

    We compare the following quadratures, all initialized with the same distribution and run independently of one another.
    In Figure~\ref{f:coverage_gibbs_vs_mcmc_c4}, we consider the history of an MH chain of length $n$ targeting $\pi$, with an isotropic Gaussian proposal tuned to reach $50\%$ acceptance, for $n \in \{100, 1000, 10000\}$.
    We then consider $n = 100$ points $x_1, \dots, x_n$ obtained via the kernel thinning algorithm of \citet{dwivedi2021generalized} with $n^2$ input points drawn again from an MH chain targeting $\pi$.
    Finally, we consider approximate samples from $\mathbb{P}_{n, \beta_n}^{V^{\pi}}$ after $T = 10\,000$ MALA updates described in Section~\ref{s:sampling} with $n = 100$ and $\beta_n = n^2$, with $M = 1\,000$ points drawn from an MH chain targeting $\pi$ to approximately compute the kernel embedding $U_{K}^{\pi}$.

    We see that $n = 100$ points approximately drawn from our Gibbs measure with temperature $\beta_n = n^2$ compare to $n = 100$ points obtained from kernel thinning in terms of coverage, and that both largely outperform crude MCMC with the same number of points.
    We would also expect from Section~\ref{s:related_work} and Corollary~\ref{cor:mixing_time_gibbs} that both kernel thinning and the Gibbs measure with $n=100$ points obtain similar coverages as the history of vanilla MCMC with $n^2 = 10\,000$ iterations.
    We see in Figure~\ref{f:coverage_gibbs_vs_mcmc_c4} that they rather match the performance of MCMC with $1\,000$ iterations. 

    The reasons can be many.
    First, one can expect Metropolis--Hastings to mix fast for such a simple choice of target $\pi$.
    Second, since we work at fixed $n$ and dimension $d = 10$, the constants in~\eqref{e:quadrature_concentration_bounded_kernel_mixing} and for kernel thinning might play an important role in the observed results.
    Finally, for the Gibbs measure, we only used a small number $T = 10\,000$ of MALA iterations, and a small number $M = 1\,000$ of MCMC input points.

    We particularly highlight the fact that we are able to match the results of kernel thinning in terms of coverage, for $\beta_n = n^2$ and with $10$ times fewer input points, since we only chose $M = 1\,000$.
    These results have been obtained using CPUs with 2 GB of RAM.
    Each run of kernel thinning with $n = 100$ points took $15$ seconds in average, \rev{using the \emph{kt.thin} function of the package \emph{goodpoints}\footnote{\href{https://github.com/microsoft/goodpoints}{https://github.com/microsoft/goodpoints} } with option \emph{store\_K} = True}.
    On the other hand, our MALA iterations to approximately sample from our Gibbs measure took around $65$ seconds per run.
    Further in favor of kernel thinning, we note that computationally cheaper --\rev{though maybe more involved to describe-- versions of kernel thinning have been recently designed \citep{li2024debiaseddistributioncompression,shetty2022distributioncompressionnearlineartime}. 
    They are also available in the \emph{goodpoints} package.}
    We recall again that the main bottleneck in the setting of expensive integrands is the number $n$ of evaluations of the integrand $f$, so that runtimes are secondary.

    In Figure~\ref{f:coverage_mh_vs_mala_c4}, we compare again the coverages obtained for $n = 100$ points approximately drawn from the Gibbs measure $\mathbb{P}_{n, \beta_n}^{V^{\pi}}$, using either MALA~\eqref{e:mala_proposal} or a simple random-walk Metropolis--Hastings with Gaussian proposal and target $\mathbb{P}_{n, \beta_n}^{V^{\pi}}$.
    We do so for the two temperature schedules $\beta_n = n^2$ and $\beta_n = n^3$, each time approximating the potential $U_{K}^{\pi}$ with $M = 1\,000$ MCMC points.

    At temperature $\beta_n = n^2$, we see that MALA performs slightly better than random walk MH, which is to be expected since the proposal~\eqref{e:mala_proposal} is informed by the log density of $\mathbb{P}_{n, \beta_n}^{V^{\pi}}$.
    We see that the coverage obtained in Figure~\ref{f:coverage_mh_vs_mala_c4} slightly improves when we increase $\beta_n = n^2$ to $\beta_n = n^3$ for points approximately drawn using random-walk MH. 
    This is what we expect after a sufficient number of iterations according to Corollary~\ref{cor:mixing_time_gibbs}.
    This is however not the case when using MALA, which means that more iterations are indeed needed in that case when going from $\beta_n = n^2$ to $\beta_n = n^3$ to produce good approximate samples from $\mathbb{P}_{n, \beta_n}^{V^{\pi}}$.
    Thus, while a reasonable dependence of the mixing time \eqref{e:mixing_time_tv_gibbs} in $\beta_n$ can be expected for random-walk MH, this might not be the case with MALA.
    We will discuss ways to improve MALA in Section~\ref{s:discussion}, along with other choices of Markov kernels that might help sampling.
    Note also that each run to approximately draw from our Gibbs measure with random-walk Metropolis--Hastings took around $15$ seconds on average, which is much cheaper than with MALA and comparable to the vanilla version of kernel thinning that we used.
    This is due to the fact that the proposal~\eqref{e:mala_proposal} requires additional computation of the gradient of $H_n$ at each step.

\subsection{Comparing worst-case errors}
In this experiment, we take for $\pi$ the uniform measure on the unit ball of $\mathbb{R}^{d}$, with $d = 3$. 
We compare, for various values $n$ of the number of quadrature nodes, the worst-case integration error $I_K$ of $(i)$ the empirical measure $\mu_n^{\mathrm{MCMC}}$ of an MH chain of length $n$ targeting $\pi$, with an isotropic Gaussian proposal with variance $0.05 I_d$, and of $(ii)$ the empirical measure $\mu_n$ of an approximate sample of $\PP$.
For the latter, we run MALA for $T=5000$ iterations. 
We consider two interaction kernels, the Gaussian kernel and the truncated Riesz kernel, 
$$
K_{1}(x, y) = \exp\left(- \lvert x  - y\vert^{2}/2\right), \text{ and } K_{2}(x, y) = \left(\lvert x - y \rvert^{2} + 0.1^{2}\right)^{-(d-2)/2}.
$$
For $K\in\{K_1, K_2\}$, we set $V$ to $V^\pi$ as in Proposition~\ref{p:inv}, with $U_{K}^{\pi}$ replaced by $$M^{-1}\sum_{i = 1}^{M} K(\cdot, z_i)$$, where $z_{1}, \dots, z_{M}$ are a sample from an MH chain of length $M = 1\,000$ targeting $\pi$, independent from any other sample.

Figure~\ref{f:three_energies_var:energy_gauss} and \ref{f:three_energies_var:energy_riesz} show the results for $K_1$ and $K_2$, respectively. 
For each value of $n$, we plot an independent approximation of $I_K(\mu-\pi)$, again approximated using long MH chains targeting $\pi$ in~\eqref{e:mmd_decomposition}, for $\mu \in \{\mu_n^{\mathrm{MCMC}}, \mu_n\}$ the empirical measure of either the baseline MH chain or an approximate draw of our Gibbs measure.
These experiments were obtained using GPU running time on Google Colab, with 12 GB of RAM.

\begin{figure}
    \centering
    \subfloat[Gaussian kernel\label{f:three_energies_var:energy_gauss}]{%
        \includegraphics[width=\threefig]{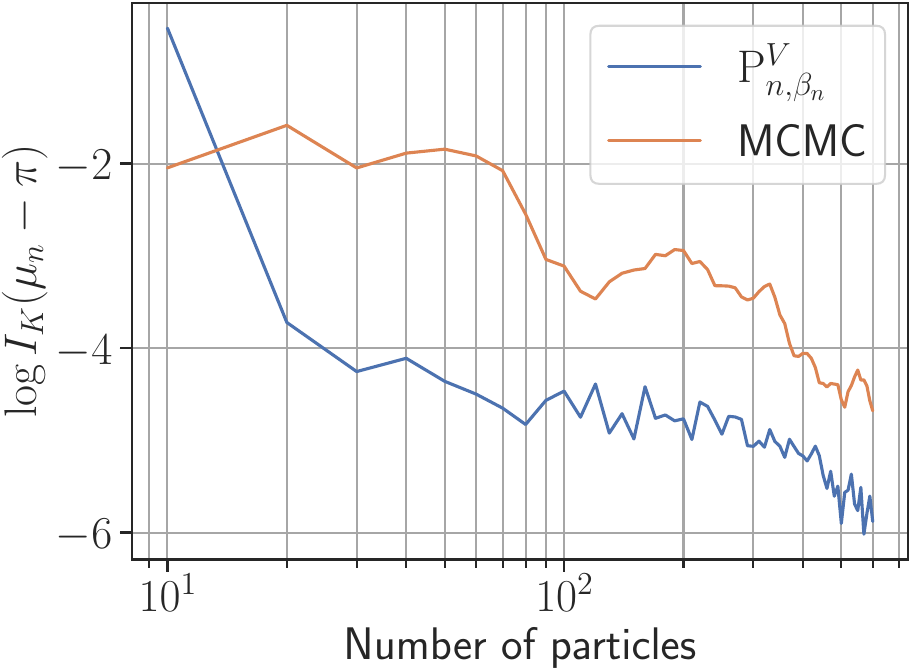}
    }
    \subfloat[Truncated Riesz kernel\label{f:three_energies_var:energy_riesz}]{
        \includegraphics[width=\threefig]{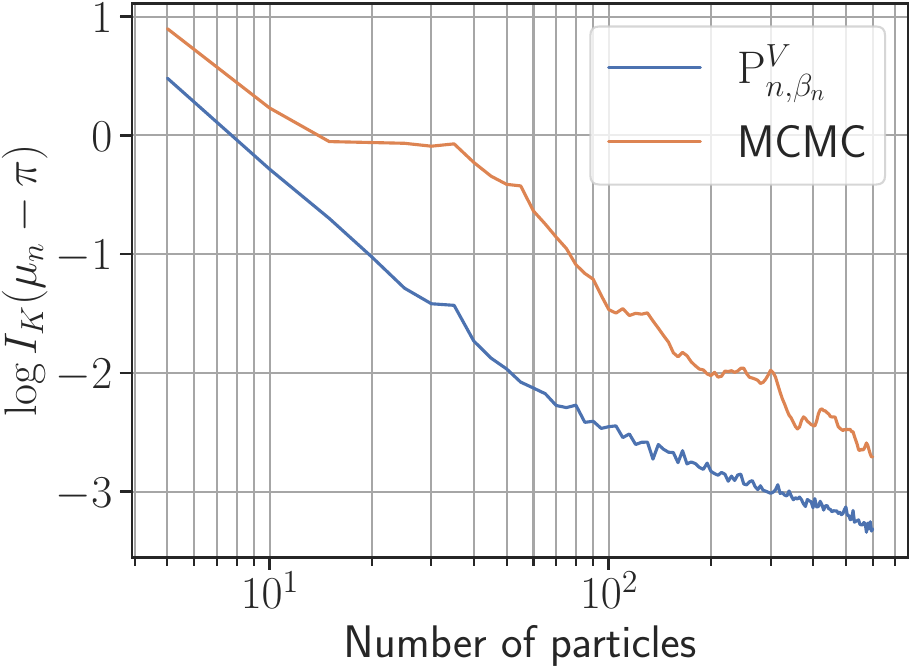}
    }
    \subfloat[Variance (truncated Riesz kernel)\label{f:three_energies_var:var_riesz}]{
        \includegraphics[width=\threefig]{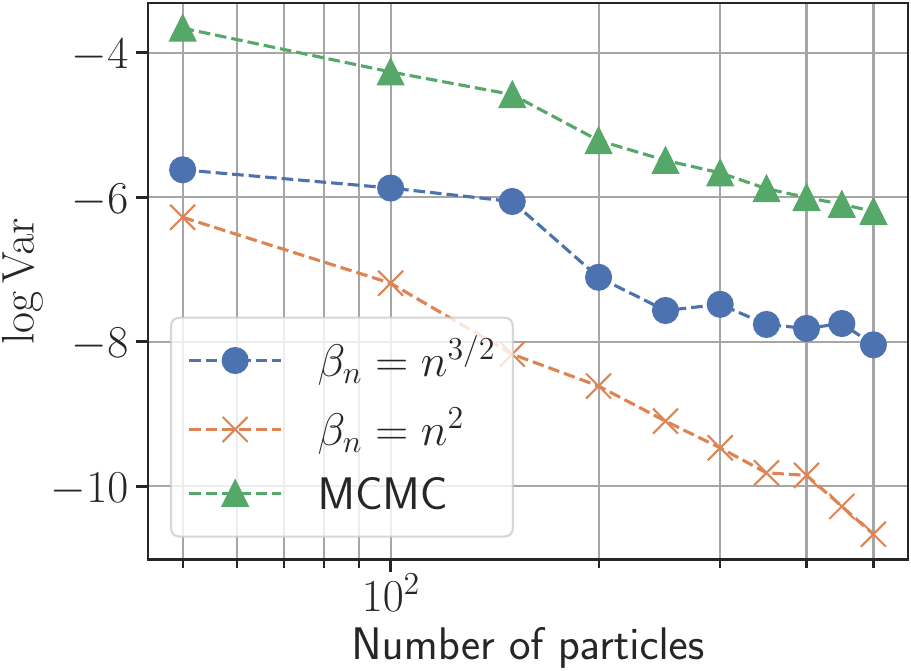}
    }
    \caption{Energy and variance comparisons between $\PP$ and MCMC samples.}
    \label{f:three_energies_var}
\end{figure}

We observe that $I_{K}(\cdot - \pi)$ decays at the same rate under the approximation of $\PP$ as for MCMC samples.
This was expected, since our concentration bound \eqref{c:sub_Gaussian_decay} recovers the $n^{-1}$ rate for the energy, the improvement rather being on the sub-Gaussian decay.
We see nonetheless that the approximated energy (and hence, the worst-case integration error) is consistently smaller by about a factor $3$ under $\PP$, which is also expected since $\PP$ favors small values of $I_K$ by definition.
We give a similar experiment in dimension $d = 10$ for a non-uniform target $\pi$ in Section~\ref{a:energy_10d}.

\subsection{Comparing variances for a single integrand}
    We know from classical CLT arguments that when $x_1, \dots, x_n$ are drawn from an MCMC chain targeting $\pi$, the variance of $n^{-1}\sum_{i = 1}^{n} f(x_i)$ scales like $n^{-1}$, under appropriate assumptions on $f$.
    While this is not at all implied by our Corollary~\ref{c:sub_Gaussian_decay}, analogies with the statistical physics literature make us expect a CLT to hold for our Gibbs measure, at rate ${\beta_n}^{-1/2}$, at least for some temperature schedules $(\beta_n)$ and smooth enough integrands.
    Such a result would imply asymptotic confidence intervals for single integrands of width decreasing like $\beta_n^{-1/2}$, a faster decay than standard Monte Carlo.
    To assess whether this expectation is reasonable, we consider again the setting where $\pi$ is uniform on the unit ball in $d=3$, the kernel is the truncated Riesz kernel $K_2$, and the integrand is $f : x \longrightarrow K(x, 0)$, which naturally belongs to $\mathcal{H}_K$.
    For each value of $n$ in Figure~\ref{f:three_energies_var:var_riesz}, we run $100$ independent MH chains of length $n$ targeting $\pi$, and plot the empirical variance of the $100$ ergodic averages. 
    Similarly, we run $100$ independent MALA chains targeting $\PP$, for $T = 5000$ iterations each, and plot the empirical variance of our estimator across these $100$ approximate samples.
    Again, $U_{K_{2}}^{\pi}$ is approximated through long MH chains. 
    These experiments were obtained using GPU running time on Google Colab, with 12 GB of RAM.
    We observe in Figure~\ref{f:three_energies_var:var_riesz} that the variance is noticeably smaller under the Gibbs measure, for both temperature schedules.
    Moreover, the rate of decay appears faster, at least in the ``usual" temperature schedule $\beta_n = n^{2}$.
    It is hard to be more quantitative, as the fact we use MALA, and with a fixed number of iterations across all values of $n$, may impact the convergence we see here.
    Still, the experiment supports our belief that a fast CLT holds, a long-time goal of the herding literature.

\section{Discussion}
\label{s:discussion}
Using a Gibbs measure that favors nodes that repel according to a kernel $K$, we improved on the non-asymptotic worst-case integration guarantees of crude Monte Carlo, through a concentration inequality with a fairly easy proof, at least compared to the classical results in statistical physics that inspired us.
    For a given precision, drawing points from this joint distribution leads to tighter confidence intervals than classical Monte Carlo, while --for now, and unfortunately-- preserving the same $n^{-1/2}$ rate for the MMD.
    We illustrated this improvement on toy examples in moderate dimension, and compared with state-of-the-art methods such as kernel thinning.

    In terms of future work, a strong argument in favor of a Gibbs measure would be a central limit theorem with a fast rate.
    As a comparison, though the fluctuations are not yet fully understood, it was shown in \citep{dwivedi2021generalized} that kernel thinning achieves a fast $n^{-1}\log n$ integration rate for any single integrand in the corresponding RKHS, a quadratic improvement over MCMC.
    We hope to match that rate \emph{and} characterize the shape of the fluctuations.
We showed experimental evidence that supports our intuition that a fast CLT can hold for our Gibbs measure. 

Two limitations of our approach that deserve further inquiry are the impact of using an approximation of the kernel embedding $U_{K}^{\pi}$ and an MCMC sampler, here MALA.
Integrating a tractable approximation of $U_{K}^{\pi}$ without loss on the convergence speed would be an important improvement.
For instance in dimension $1$, a stability result for the equilibrium measure is shown in \citep{maida2014free}, with respect to a modification of the external confinement $V$ which in our case amounts to the kernel embedding.
Simultaneously, understanding how the accuracy of the MALA approximation relates to $n$ and $\beta_n$ would help find the right trade-off between statistical accuracy and computational cost, \rev{following Corollary~\ref{cor:mixing_time_gibbs}}.
For that purpose, we note that mixing time guarantees have been obtained in non-convex settings in \citep{ma2019sampling}, which support our hope to obtain polynomial mixing times for a MCMC kernel tailored to our Gibbs measure, at least in the regime $\beta_n = \mathcal{O}(n^2)$.
\rev{At lower temperatures, we expect the sampling to get harder. Apart from our numerical observations in Section~\ref{s:experiments}, this is also related to the fact that achieving a polynomial mixing time for such Gibbs measures on the sphere with $\beta_n = n^3$ would give a probabilistic answer to Smale's 7th problem, see \citep{BeHa19}}.
Moreover, it might be possible to explain the difficulty to sample at low temperatures (e.g., $\beta_n=n^3$) by lower bounding an escape probability as in \citep[Theorem 3.4]{gamarnik2023algorithmic}.
Yet in a different line of work, there are positive results that reduce sampling from Gibbs measures to sampling independent sets in a random graph \citep{friedrich2025sampling}, leading to an approximate sampler in polynomial time under strong conditions on the interaction potential.

Finally, compared to the bound \eqref{e:ayoub_concentration} for volume sampling, our concentration bound features $\beta_n$, but not the eigenvalues of the kernel operator. 
In that sense, our confidence intervals are likely to be looser in an RKHS with fast-decaying kernel eigenvalues.
Yet, establishing a CLT for the estimator in \eqref{e:ayoub_concentration}, where the weights in the estimator depend on all quadrature nodes, promises to be much harder than for our Gibbs estimator. 
    Moreover, one evaluation of our Gibbs density is quadratic in $n$, while it is cubic for volume sampling, a fact that may influence the total cost of approximate sampling using MCMC.

\section{Proofs}
\label{s:proofs}

\subsection{Proof of Proposition~\ref{p:eq}}

We first recall the result.
 \begin{proposition}
Let $K$ satisfy Assumptions~\ref{a:bounded_kernel} and \ref{a:isdp_kernel}, and $V$ satisfy Assumption~\ref{a:confinement}. Then 
\begin{enumerate}
    \item $I_{K}^{V}$ is lower semi-continuous, has compact level sets and 
    $
         I_{K}^{V}(\mu) > - \infty
    $
    for any probability distribution $\mu$ on $\mathbb{R}^d$.
    \item If $\mu \in \mathcal{E}_{K}^{V}$, then $I_{K}(\mu)$ and $\int \lvert V\rvert\,\mathrm{d}\mu$ are finite, and 
    $
        I_{K}^{V}(\mu) = \frac{1}{2}I_{K}(\mu) + \int V\,\mathrm{d}\mu.
    $
    \item $I_{K}^{V}$ is strictly convex on the convex non-empty set $\mathcal{E}_{K}^{V}$.
    \item $I_{K}^{V}$ has a unique minimizer $\mu_V$ over the set of probability measures on $\mathbb{R}^{d}$, called the \emph{equilibrium measure}, and the support of $\mu_V$ is compact.
\end{enumerate}
\end{proposition}
The proof is a simple application of known results.
\begin{enumerate}
    \item This is a consequence of the first point of Theorem $1.1$ in \citep{chafai2014first};
    \item This is given by the second point of Lemma $2.2$ in \citep{chafai2014first};
    \item This is a consequence of the ISDP assumption on the kernel, using Lemma $3.1$ in \citep{pronzato2020bayesian}.
    To see that $\mathcal{E}_{K}^{V}$ is non-empty, we can simply consider Dirac measures $\delta_x$, since $K$ is bounded and $V$ is finite everywhere.
    \item This is a consequence of points $1$ and $3$ using the arguments of Section $4.1$ of \citep{chafai2014first}.
\end{enumerate}

\subsection{Proof of Proposition~\ref{p:inv}}
The proof of Proposition~\ref{p:inv} relies on the so-called Euler-Lagrange equations, which we recall here.

\begin{lemma}\label{p:eul}
Let $K$ be a kernel satisfying Assumptions~\ref{a:bounded_kernel} and \ref{a:isdp_kernel}, and $V$ be an external potential satisfying Assumption~\ref{a:confinement}.
Set $C_V = I_{K}(\mu_V) + \int V\,\mathrm{d}\mu_V$.
Then $\mu_V$ has compact support, and a probability measure $\nu$ satisfies $\nu = \mu_V$ if and only if $\nu$ has compact support and there exists a constant $C$ such that 
\begin{itemize}
    \item[(i)] $U_{K}^{\nu}(z) + V(z) \geq C$ for all $z \in \mathbb{R}^{d}$;
    \item[(ii)] $U_{K}^{\nu}(z) + V(z) \leq C$ for all $z \in supp(\nu)$.
\end{itemize}
In that case, $C = C_V$.
\end{lemma}

\begin{proof}
    Let us first check that $\mu_V$ satisfies this characterization.
    We already know from Proposition~\ref{p:eq} that $\mu_V$ exists, is unique, and has compact support.
    We can use the same procedure as \citet[Theorem $1.2$, proof of item $5$]{chafai2014first} : considering the directional derivative of $I_{K}^{V}$ and using the fact that $\mu_V$ is the minimizer, their equation ($4.5$) yields that, for any probability distribution $\nu \in  \mathcal{E}_{K}^{V}$, 
    \[
        \int \left(V + U_{K}^{\mu_V} - C_V\right)\,\mathrm{d}\nu \geq 0.
    \]
    Since Dirac measures $\delta_z$ have finite interaction energy $I_{K}^{V}(\delta_z)$, we get Point $(i)$ by taking $\nu = \delta_z$.
    The second point is obtained exactly as in the second part of the proof of item $5$ of Theorem $1.2$ of \citep{chafai2014first}.
    Finally, the converse implication can be similarly obtained along the lines of the proof of item $6$ of Theorem $1.2$ in \citep{chafai2014first}, using the strict convexity of the energy functional in Proposition~\ref{p:eq}.
\end{proof}

We are now ready to prove Proposition~\ref{p:inv}, by checking that the Euler-Lagrange equations are satisfied for $\pi$ and $V^{\pi}$, and that $V^{\pi}$ satisfies the assumptions of Proposition~\ref{p:eq}.

First note that since the kernel $K$ is nonnegative and bounded on the diagonal by assumption, Cauchy-Schwarz in the RKHS $\mathcal{H}_{K}$ implies that $K(x, y) \leq C$ for all $x, y \in \mathbb{R}^{d}$.
As a consequence, since $K$ is further assumed to be continuous, Lebesgue's dominated convergence theorem yields that $z \longmapsto U_{K}^{\pi}(z)$ is finite everywhere and continuous.
In particular, $V^\pi$ is continuous.

Moreover, the bound on $K$ induces $0 \leq U_{K}^{\pi}(z) \leq C$ for any $z \in \mathbb{R}^{d}$, so that $V^{\pi}(z) \longrightarrow + \infty$ when $\lvert z \rvert \longrightarrow +\infty$.
Finally, the integrability assumption of Assumption~\ref{a:confinement} is satisfied by our assumption on $\Phi$ and because $U_{K}^{\pi}$ is bounded. 
Hence, $V^{\pi}$ satisfies Assumption~\ref{a:confinement} and the equilibrium measure $\mu_{V^{\pi}}$ is well-defined.
We conclude upon noting that, since $\Phi \geq 0$, Proposition~\ref{p:eul} yields that $\mu_{V^{\pi}} = \pi$.

\subsection{Proof of Theorem~\ref{th:concentration_inequality}}

The proof will follow the one of \citet{chafai2018concentration}, with some notable simplifications.
We first compute a lower bound on the partition function $Z_{n, \beta_n}^{V}$, which generalizes the one of \citet{chafai2018concentration} to bounded kernels.

\begin{proposition}
    \label{p:lower_bound_on_the_normalization_constant}
    Let $K$ be a kernel satisfying Assumptions~\ref{a:bounded_kernel} and~\ref{a:isdp_kernel}, and $V$ be an external potential satisfying Assumption~\ref{a:confinement}.
    Assume that the associated equilibrium measure $\mu_V$ has finite entropy, i.e. $S(\mu_V) = -\int \log\mu'_V \d\mu_V <\infty$, where $\mu_V'$ is the density of $\mu_V$ w.r.t. the Lebesgue measure.
    Then for $n \geq 2$, we have
    \[
        Z_{n, \beta_n}^{V} \geq \exp\left\{- \beta_n I_{K}^{V}(\mu_V) + n\left(\frac{\beta_n}{2 n^2} I_{K}(\mu_V) + S(\mu_V)\right)\right\}.
    \]
\end{proposition}

\begin{proof}
The idea is to rephrase a bit the discrete energy $H_n$ and use Jensen's inequality.

We let $X_n = (x_1, \dots, x_n)$ for brevity, and we start by writing 
\[
    n^2 H_{n}(X_n) = \frac{1}{2} \sum_{i \neq j}\{ K(x_i, x_j) + V(x_i) + V(x_j)\} + \sum_{i = 1}^{n} V(x_i).
\]
Then
\begin{align*}
    \log  Z_{n, \beta_{n}}^{V} &= \log \int_{\left(\mathbb{R}^{d}\right)^n}\exp\left( - \beta_n H_n(X_n)\right) \mathrm{d}x_1\dots\,\mathrm{d}x_n\\
    &\geq \log \int_{E_{V}^{n}} \exp\left( -\frac{\beta_n}{2 n^2} \sum_{i \neq j} \{ K(x_i, x_j) + V(x_i) + V(x_j)\}\right)\\
    & \times \exp\left(- \sum_{i = 1}^{n}\left( \frac{\beta_{n}}{n^2} V(x_i) + \log \mu'_V(x_i)\right)\right) \mathrm{d}\mu_{V}(x_1)\,\mydots\,\mathrm{d}\mu_{V}(x_n),
\end{align*}
where $E_{V}^{n} = \{(x_1, \dots, x_n) \in (\mathbb{R}^{d})^{n} : \prod_{i = 1}^{n} \mu'_V (x_i) > 0\}$. 
Using Jensen's inequality, we get 
\begin{align*}
\log Z_{n, \beta_n}^{V} &\geq -\frac{\beta_n}{ n^2} \sum_{i \neq j} \frac{1}{2} \int_{E_{V}^{n}} \left(K(x_i, x_j) + V(x_i) + V(x_j)\right) \mathrm{d}\mu_{V}(x_1)\,\mydots\,\mathrm{d}\mu_{V}(x_n)\\
& \quad - \sum_{i = 1}^{n} \int_{E_{V}^{n}} \left(\frac{\beta_{n}}{n^2} V(x_i)+\log \mu'_{V}(x_i)\right) \mathrm{d}\mu_{V}(x_1)\,\mydots\,\mathrm{d}\mu_{V}(x_n)\\
&= -\frac{\beta_n}{n^2}n(n-1) I_{K}^{V}(\mu_V) - \frac{\beta_n}{n}\int V\mathrm{d}\mu_V + n S(\mu_V).
\end{align*}
Using the definition $I_{K}^{V}(\mu_V) = \frac{1}{2} I_{K}(\mu_V) + \int V \mathrm{d}\mu_V$ we get the result.
\end{proof}

To bound $I_{K}(\mu_n - \mu_V)$, we shall further use the following lemma, which is again inspired by \citep{chafai2018concentration}.

\begin{lemma}\label{l:1}
    Let $K$ and $V$ be a kernel and an external potential satisfying Assumptions~\ref{a:bounded_kernel} to \ref{a:confinement}.
    Let $\mu$ be any probability measure of finite energy, $I_{K}^{V}(\mu)<+\infty$.
    Then 
    \[
        I_{K}(\mu - \mu_V) \leq 2\left(I_{K}^{V}(\mu) - I_{K}^{V}(\mu_V)\right).
    \]
\end{lemma}

\begin{proof}
    We write 
    \begin{equation}
        I_{K}(\mu - \mu_V) = \iint K(x, y)\,\mathrm{d}(\mu - \mu_V)^{\otimes 2} = I_{K}(\mu) - 2 I_{K}(\mu, \mu_V) + I_{K}(\mu_V).
        \label{e:decomposition}
    \end{equation}
    With the notation of Proposition~\ref{p:eul}, we know that $U_{K}^{\mu_V}(z) + V(z) = C_V$ for all $z$ in the support of $\mu_V$, and that $U_{K}^{\mu_V}(z) + V(z) \geq C_V$ in general.
    Using Fubini, we thus see that 
    \begin{align*}
        I_{K}(\mu, \mu_V) + \int V\,\mathrm{d}\mu &= \int \left\{ U_{K}^{\mu_V}(z) + V(z)\right\}\,\mathrm{d}\mu(z)\\
        &\geq C_V = I_{K}(\mu_V) + \int V\,\mathrm{d}\mu_V.
    \end{align*}
    Plugging this into \eqref{e:decomposition} yields the result.
\end{proof}

We are now ready to prove Theorem~\ref{th:concentration_inequality}.
Recall that we work under the assumption $\beta_n \geq 2 cn$ where $c$ is the constant of Assumption~\ref{a:confinement}.
Consider a Borel set $A \subset \left(\mathbb{R}^{d}\right)^{n}$. 
For brevity, recall that for $(x_k)_{k\leq n} \in A$, we write $\mu_n = \frac{1}{n}\sum_{i = 1}^{n} \delta_{x_i}$ and $X_n = (x_1, \dots, x_n)$.

Let $\eta > 0$ and $X_n\in A$.
The key idea is to split $H_n(X_n) = (1-\eta) H_n(X_n) + \eta H_n(X_n)$. 
The first part will be compared with the energy $I_{K}^{V}(\mu_n)$, while the second part will be kept to ensure integrability.
Remembering that 
\[
    n^2 H_n(X_n) = n^2 I_{K}^{V}(\mu_n) - \frac{1}{2}\sum_{i = 1}^{n} K(x_i, x_i),
\]
we write, by definition,
\begin{align}
    \mathbb{P}_{n, \beta_n}^{V}(A) &= \frac{1}{Z_{n, \beta_n}^{V}} \int_{ A} \exp\left( -\beta_n H_n(X_n)\right) \,\mathrm{d}x_1\,\dots\,\mathrm{d}x_n\\
    &= \frac{1}{Z_{n, \beta_n}^{V}} \exp\left(-\beta_n (1-\eta) I_{K}^{V}(\mu_V)\right) \int_{A} \exp\left(-\beta_n (1-\eta) \left(I_{K}^{V}(\mu_n) - I_{K}^{V}(\mu_V)\right)\right)\nonumber\\
    & \quad \times \exp\left(\frac{\beta_n}{2 n^2} (1-\eta) \sum_{i = 1}^{n} K(x_i, x_i) - \beta_n \eta H_n(X_n)\right) \,\mathrm{d}x_1\,\mydots\,\mathrm{d}x_n.
\end{align}    
Using Proposition~\ref{p:lower_bound_on_the_normalization_constant}, we continue
\begin{align}
    \mathbb{P}_{n, \beta_n}^{V}(A) 
    &\leq \exp\left(- n \left(S(\mu_V) + \frac{\beta_n}{2n^2} I_{K}(\mu_V)\right) + \beta_n \eta I_{K}^{V}(\mu_V)\right)\\
    & \quad \times \exp\left(-\beta_n (1-\eta) \underset{A}{\,\inf}\left( I_{K}^{V}(\mu_n) - I_{K}^{V}(\mu_V)\right)\right)\\ 
    & \quad \times \int_{\mathbb{R}^{d}} \exp\left(\frac{\beta_n}{2 n^2} (1-\eta) \sum_{i = 1}^{n} K(x_i, x_i) - \beta_n \eta H_n(X_n)\right) \,\mathrm{d}x_1\,\mydots\,\mathrm{d}x_n. \label{e:intermediate_step}
\end{align}
As noted in the proof of Proposition~\ref{p:eul}, $0\leq K\leq C$, so that the last integral in \eqref{e:intermediate_step} is easily bounded,
\begin{align*}
    &\int_{\mathbb{R}^{d}} \exp\left(\frac{\beta_n}{2 n^2} (1-\eta) \sum_{i = 1}^{n} K(x_i, x_i) - \beta_n \eta H_n(X_n)\right) \,\mathrm{d}x_1\,\dots\,\mathrm{d}x_n\\
    &\quad \leq \exp\left(\frac{\beta_n}{2 n} (1-\eta) C\right) \left(\int_{\mathbb{R}^{d}}\exp\left(-\frac{\beta_n}{n} \eta V(x)\right)\,\mathrm{d}x\right)^n.
\end{align*}
Now we choose a particular value for $\eta$, namely $\eta = c n / \beta_n$ where $c$ is the constant of Assumption~\ref{a:confinement}.
Further let $C_2 = \log \int_{\mathbb{R}^{d}}\exp\left(- c V(x)\right)\,\mathrm{d}x$, which is finite by assumption.
Setting $c_1 = c I_{K}^{V}(\mu_V) + C_2 - S(\mu_V)$ and $c_2 = \frac{1}{2}C - \frac{1}{2} I_{K}(\mu_V)$, \eqref{e:intermediate_step} yields 
\begin{equation}
    \label{e:almost_there}
    \mathbb{P}_{n, \beta_n}^{V}\left(A\right) \leq \exp\left(- (\beta_n - c n) \underset{A}{\,\inf}\left( I_{K}^{V}(\mu_n) - I_{K}^{V}(\mu_V)\right) + n c_1 + \frac{\beta_n}{n}c_2\right).
\end{equation}
Note that $c_1$ and $c_2$ are indeed finite by definition of $\mu_V$, since $\mathcal{E}_{K}^{V}$ is non-empty.
The comparison inequality of Lemma~\ref{l:1} applied to \eqref{e:almost_there}, with
\[
    A = \{I_{K}(\mu_n -\mu_V) > r^2 ; \mu_n = \frac{1}{n}\sum_{i = 1}^{n} \delta_{x_i}\},
\]
yields Theorem~\ref{th:concentration_inequality}, upon noting that $\beta_n - c n \geq \beta_n /2 $ by assumption.

\subsection{Proof of Proposition~\ref{l:bounds_constants_concentration}}
\begin{proof}
    Using the fact that $c = 1$ from Proposition~\ref{p:inv} and
    \begin{align}
    I_{K}^{V^{\pi}}(\pi) &= \frac{1}{2}\iint K(x, y)\,\mathrm{d}\pi(x)\,\mathrm{d}\pi(y) - \iint K(x, y)\,\mathrm{d}\pi(x)\,\mathrm{d}\pi(y) + \int_{\lvert x\rvert > R} \Phi(x) \,\mathrm{d}\pi(x)\nonumber\\
    &= - \frac{1}{2}I_{K}(\pi),
    \end{align}
    we first have
    \begin{align}
        \label{e:constants_concentration}
        &c_1 = - \frac{1}{2}I_{K}(\pi) - S(\pi) + \log\int_{\mathbb{R}^{d}} \exp\left(- V^{\pi}(x)\right)\,\mathrm{d}x,\\
        &c_2 = \frac{1}{2} C - \frac{1}{2}I_{K}(\pi).
    \end{align}
    Then, we use the fact that $0 \leq K(x, y)\leq C$ for all $x, y \in \mathbb{R}^{d}$ to get $I_{K}(\pi) \geq 0$ and $c_2 \leq C/2$.
    Moreover, using the definition of $V^{\pi}$ in Proposition~\ref{p:inv}, we have the bound
    \begin{align}
    \int_{\mathbb{R}^d} \exp\left(-V^{\pi}(z)\right)\,\mathrm{d}x \leq \int_{\lvert x\rvert < R}\,\exp(C)\mathrm{d}x + \int_{\lvert x\rvert > R}\exp(C)\exp\left(-\lvert x\rvert^2 + R^2\right)\,\mathrm{d}x,
    \end{align}
    since $U_{K}^{\pi}(x) \leq C$ for all $x\in \mathbb{R}^{d}$.
    For the first term, we use the fact that the volume of $B(0, R)$ is given by
    \begin{align}
    \left\lvert B(0, R)\right\rvert = \frac{\pi^{d/2}}{\Gamma\left(d/2+1\right)}R^d.
    \end{align}
    For the second term, we bound the tail of the Gaussian distribution
    \begin{align}
    \int_{\lvert x\rvert > R} \exp\left(-\lvert x\rvert^2\right)\,\mathrm{d}x \leq \int \exp\left(-\lvert x\rvert^2\right)\,\mathrm{d}x \leq \left(\sqrt{2}\pi\right)^{d/2}.
    \end{align}
    This yields
    \begin{align}
    c_1 \leq - S(\pi) + 2C + d\left(\frac{(1+\sqrt{2})\log\pi}{2} + \log R\right) - \log \Gamma\left(d/2+1\right)+R^2.
    \end{align}
\end{proof}

\section*{Acknowledgments}
This work was funded by ERC grant \textsc{Blackjack} ERC-2019-STG-851866, ANR grant \textsc{Baccarat} ANR-20-CHIA-0002 and Labex \textsc{CEMPI} ANR-11--LABX-0007-01.

\printbibliography
\appendix

\section{A note on the physics-inspired vocabulary}
\label{app:physics_vocabulary}
Gibbs measures are intimately connected to statistical physics, and it helps the intuition to have in mind the physical situation behind the names, for instance, of the different notions of energy introduced in Definition 3.1. 
In a world where the Coulomb interaction is given by $K$, $U_{K}^{\mu}(z)$ would be the electric potential created at point $z$ by charges distributed according to $\mu$.
In the same way, $I_{K}^{V}(\mu)$ is the energy of charges distributed according to $\mu$, repelling each other according to $K$, and confined by some external potential $V.$ 



\section{Additional discussion and experiments}
\label{a:comput}

\subsection{Computational clarifications on sampling from $\PP$}
\label{a:comput:complexity}

To clear one's mind, we first discuss how our approximate sampling algorithm for $\PP$ runs compared to vanilla MCMC with target $\pi$.
We approximatively sample from $\PP$ by running a number $T$ of MALA iterations, and consider the last $T-$th iteration as our approximate sample $(y_i)_{1 \leq i \leq n}$, where each $y_i$ belongs to $\mathbb{R}^{d}$.
On the other hand, the vanilla MCMC to which we compare runs a Markov chain on $\mathbb{R}^{d}$ with target $\pi$ and $n$ iterations, and averages over the chain to approximate the target integral, thus using $n$ quadrature nodes in $\mathbb{R}^{d}$ as well.
Of course, our MALA chain runs in much higher dimension ($d\times n$), but we emphasize that we are not comparing the convergence of this MALA chain to $\PP$ versus vanilla MCMC with target $\pi$, but rather how our estimator based on a single sample from $\PP$ (obtained after $T \gg n$ MALA iterations) compares to the estimator averaging over the $n$ vanilla MCMC iterations.

 Our procedure clearly requires more CPU time, yet what we are comparing is the performance of both rules as a function of $n$. 
This is again fair in the setting of expensive-to-evaluate integrands $f$, where the computational bottleneck is the number $n$ of times that we have to evaluate $f$ and one can usually afford a larger CPU time to pick the nodes. 
The important point is that the output $n$ nodes result in a confidence interval for integral approximation that is smaller than with $n$ ``classical" MCMC iterations, which is what our results claim.

Expensive-to-evaluate integrands is a standard setting for deterministic kernel-based particle methods such as kernel herding or Stein variants, where the algorithmic complexity of each of the $T \gg n$ iterates is often in $\mathcal{O}(n^2)$ \citep{korba2021kernel, chen2018steinpoints, chen2019steinmcmc}. Our approach also runs in $\mathcal{O}(n^2)$ per iteration since we need to compute pairwise interactions at each update.

\subsection{Approximation of the kernel embedding}
\label{a:comput:kernel_approx}

Evaluating the kernel embedding is key to design the Gibbs measure $\mathbb{P}_{n, \beta_n}^{V^{\pi}}$ that concentrates around $\pi$. 
Since $U_{K}^{\pi}(z)$ does not usually have a closed form, a standard way \citep{pronzato2020bayesian} to approximate it is to run a long Metropolis--Hastings chain of length $M$ targeting $\pi$ and average $K(z, .)$ along the chain.
Recall that this is fair for expensive-to-evaluate target integrands $f$, since no additional evaluation of $f$ is involved.
From the computational point of view, since any evaluation of $U_{K}^{\pi}$ needs to be replaced by a sum of size $M$, the computational complexity of each iteration of MALA becomes $\mathcal{O}(n^2 + nM)$.

If one insists on avoiding having to approximate $U_{K}^{\pi}$, the Stein kernel $K_{\pi}$ associated to $K$ and $\pi$ has emerged as a tool to bypass this approximation \citep{anastasiou2023stein, oates2017control} when $\pi$ is either fully supported on $\mathbb{R}^{d}$, or when $\pi$ has compact support up to a smooth modification of the kernel $K$ \citep{oates2019convergence}, so that $U_{K_{\pi}}^{\pi}(z) = 0$ for $\pi-$almost all $z$ .
Other techniques in the latter case involve mirrored Stein operators \citep{shi2021sampling} or a mollified version of the energy \citep{li2023sampling}.

Nonetheless, changing $K$ for the induced Stein kernel changes the RKHS of integrands in a nontrivial way, which is not desirable for numerical integration in a RKHS with given kernel $K$.
In the same vein, \citep{benard2024kernel, korba2021kernel} report unintuitive boundary effects as well as biased weights when the target is a simple mixture of fully supported Gaussians, along with some practical fixes.

We thus stick in this paper to the standard Monte Carlo approximation for $U_{K}^{\pi}$. While incorporating this approximation in our theoretical results would be an important (but challenging) feature, we can only state that our experiments on toy models empirically validate this approximation, since we recover known statistical properties of the target Gibbs measure; see Figure~\ref{f:three_energies_var} with a reasonably small number of Monte Carlo iterations $M$.

\subsection{Multimodal example}
\label{a:comput:multimodal}

We set $\pi$ to be a balanced mixture of six two-dimensional truncated Gaussian distributions with radius $0.5$, variance $0.1$ and whose centers are evenly spaced on the unit circle. We stick to the truncated logarithmic interaction kernel as in the first experiment of Section 6, with $n = 1\,000$ and $\beta_n = n^{2}$, while the kernel embedding is approximated with $M = 10\,000$ iterations of MCMC as explained in Section~\ref{a:comput:kernel_approx}.
Such multimodal target distributions are simple challenging toy examples in which one might suffer from a cold start and slow mixing.

If we take $n$ large enough and
an initial distribution with a large support, we can expect a few of the $n$ points to be close to each
mode of $\pi$, and the MALA chain to keep an appropriate proportion in each mode of $\pi$.
However, as for any optimization algorithm targeting an objective like a worst-case integration error, we could suffer from a bad initialization with no particles in an isolated mode of $\pi$, and a large number of MALA iterations might be needed to reach a good approximation of $\PP$ and thus $\pi$; cf. the discussion of Section 6.1.

In Figure~\ref{f:mixture_cold}, we plot the state of the MALA chain at three different iteration numbers, for a cold Gaussian start with mean $0$ and variance $1$. 
We then consider in Figure~\ref{f:mixture_warm_ksd} a warm start obtained through $50,000$ iterations of KSD descent \citep{korba2021kernel}, with Gaussian kernel and whose target is a mixture of Gaussian distributions centered in each of the modes of $\pi$ with variance $0.1$.
Another classical complementary solution, also mentioned by \citet{korba2021kernel} in failure cases (and especially for multimodal targets), is to rely on annealing.
To wit, we introduce a sequence of tempered targets $\pi_k(x) \propto e^{- t_k \log \pi(x)}$, for a sequence $(t_k)_{1 \leq k \leq l}$ going from $0.1$ to $1$.
We successively use the result of the sampling algorithm on $\pi_{k}$ as the initialization for the sampling algorithm on $\pi_{k+1}$, up to $\pi_{l} = \pi$. 
We choose here a naive annealing scheme $l = 10$ and $t_k = k/10$.
We then run our MALA updates \emph{after} $50,000$ iterations of KSD descent for which annealing on $\pi$ is used (\ref{f:mixture_warm_ksd_annealing}), as well as MALA updates from a cold Gaussian start, but progressively annealing $\pi$ in the MALA updates (\ref{f:mixture_cold_annealing}).
In detail, for each $k$, we approximate the kernel embedding of $\pi_{k}$ and then run MALA (for $T/10$ iterations, to keep our global number of MALA iterations constant) targeting the associated Gibbs measure as usual, and use the last iteration to initialize the same MALA updates for $\pi_{k+1}$, up to $\pi_{l} = \pi$ where the target measure of MALA becomes $\PP$.

\begin{figure}
    \centering
    \subfloat[$T = 5,000$]{%
        \includegraphics[width=\threefig]{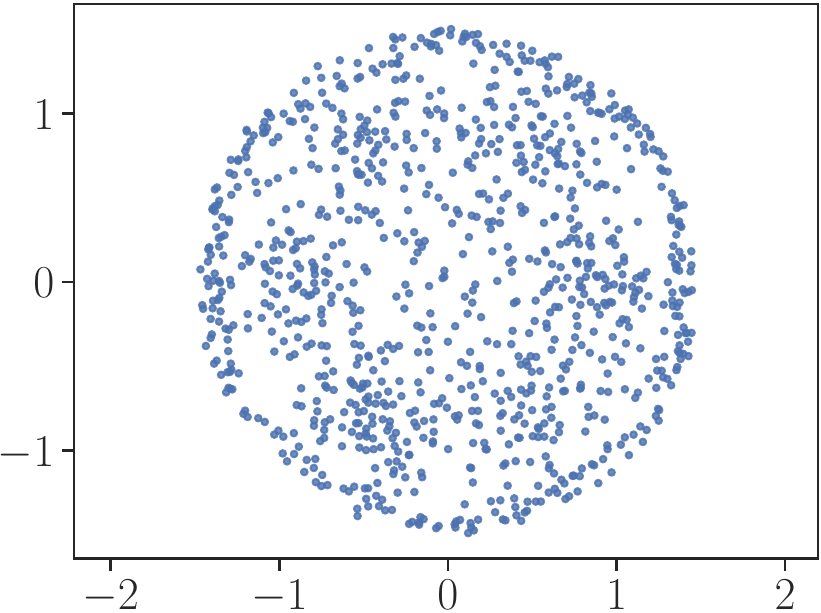}
    }
    \subfloat[$T = 10,000$]{
        \includegraphics[width=\threefig]{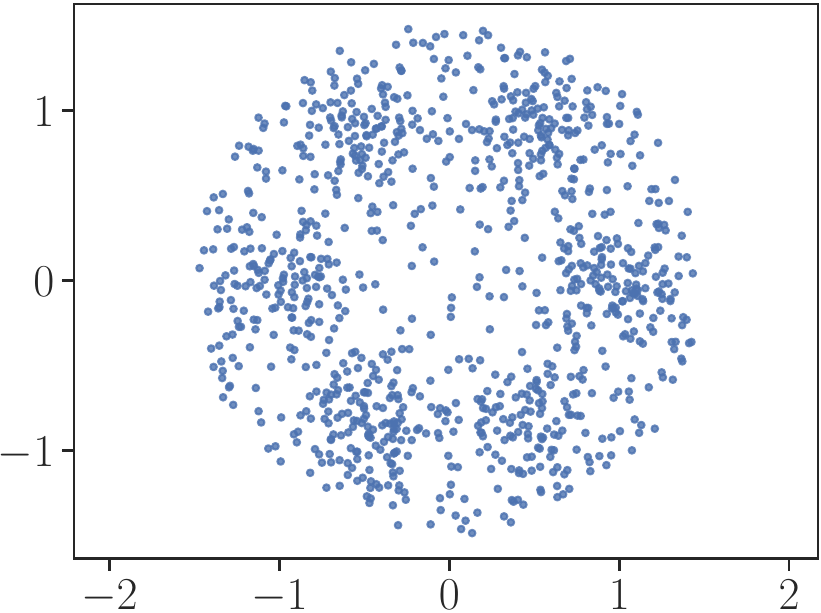}
    }
    \subfloat[$T = 15,000$]{
        \includegraphics[width=\threefig]{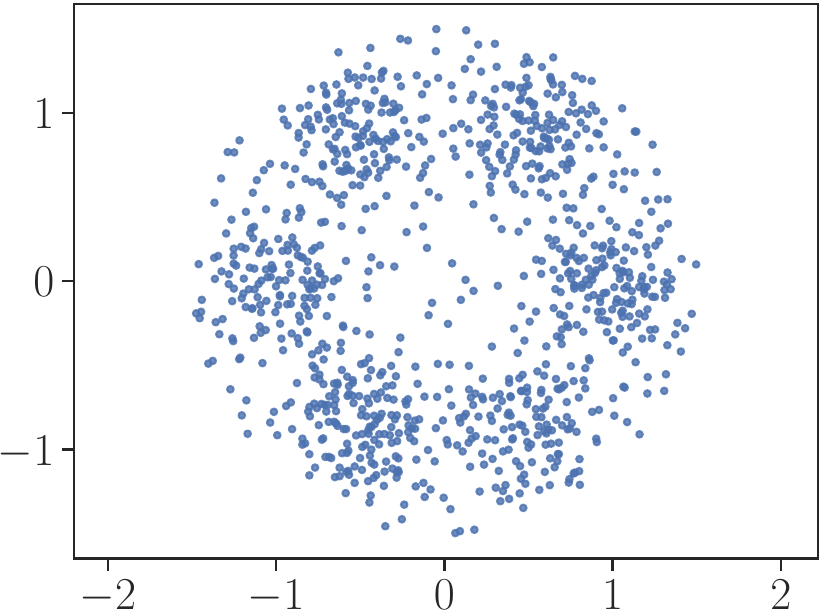}
    }
    \caption{Cold Gaussian start}
    \label{f:mixture_cold}
\end{figure}

\begin{figure}
    \centering
    \subfloat[$T = 5,000$]{%
        \includegraphics[width=\threefig]{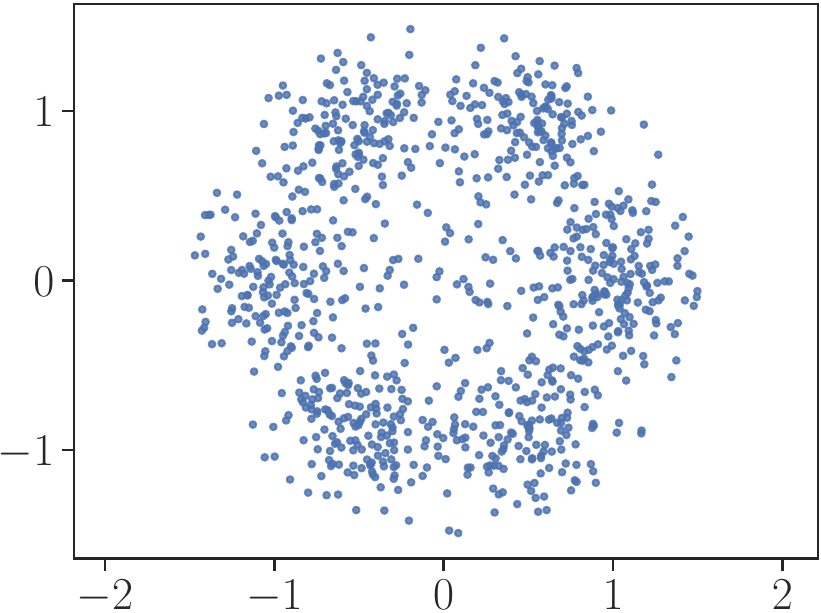}
    }
    \subfloat[$T = 10,000$]{
        \includegraphics[width=\threefig]{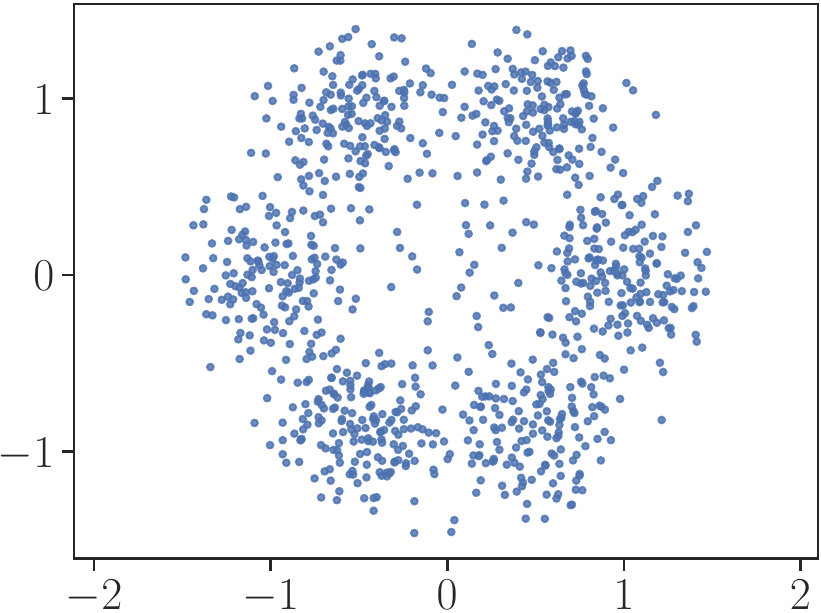}
    }
    \subfloat[$T = 15,000$]{
        \includegraphics[width=\threefig]{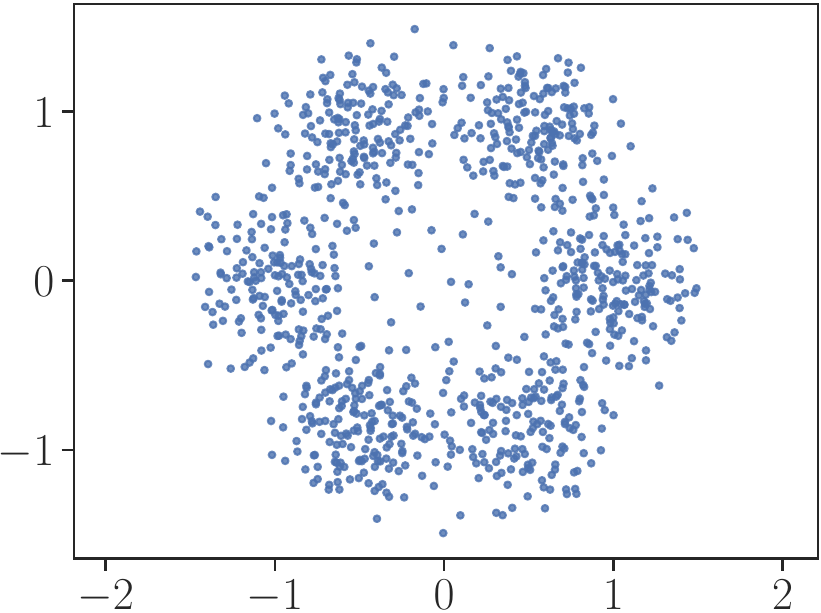}
    }
    \caption{Warm KSD descent start}
    \label{f:mixture_warm_ksd}
\end{figure}

\begin{figure}
    \centering
    \subfloat[$T = 5,000$]{%
        \includegraphics[width=\threefig]{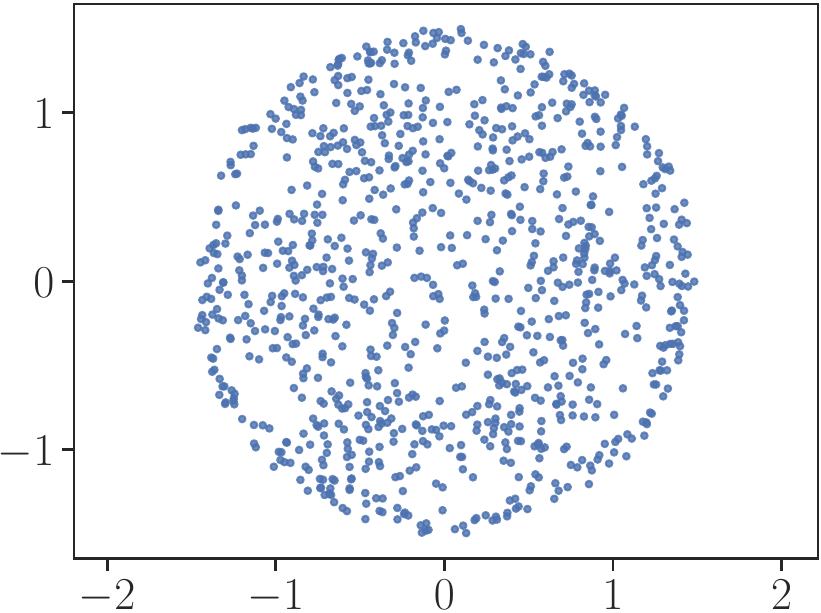}
    }
    \subfloat[$T = 10,000$]{
        \includegraphics[width=\threefig]{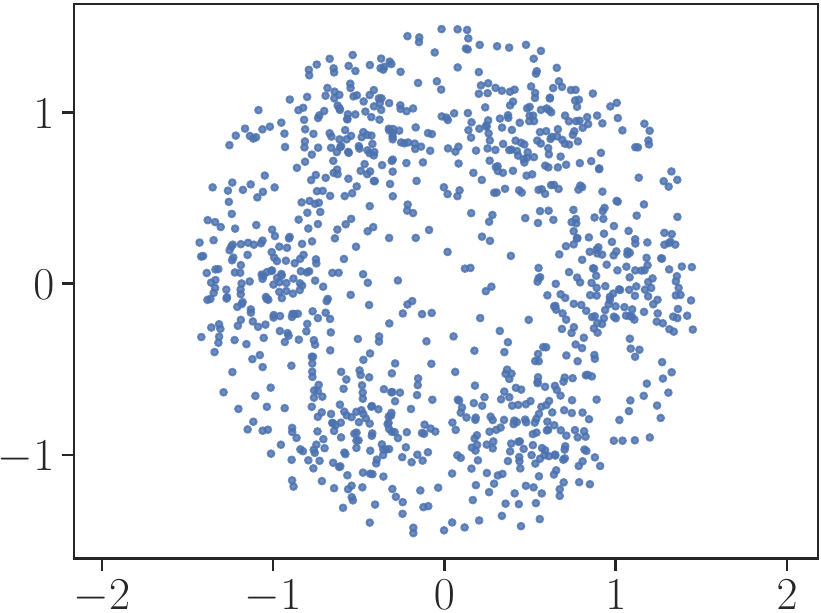}
    }
    \subfloat[$T = 15,000$]{
        \includegraphics[width=\threefig]{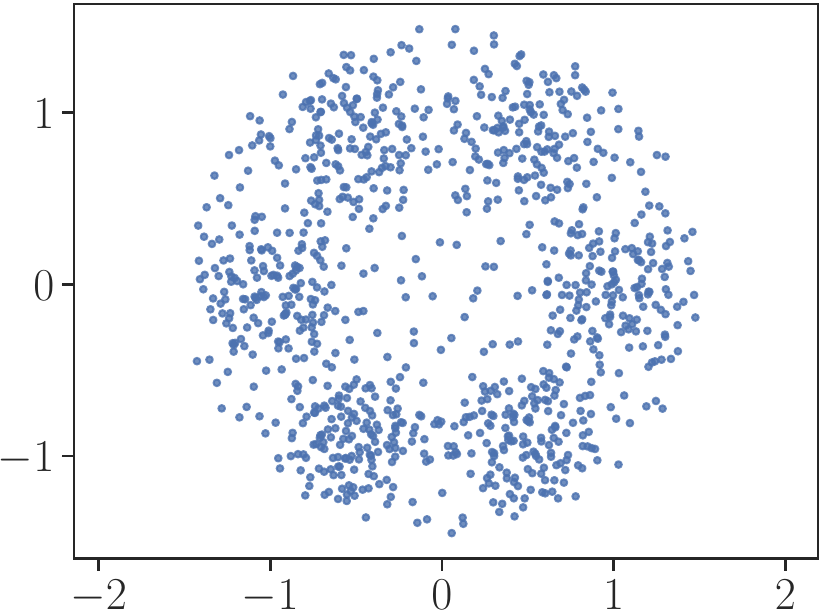}
    }
    \caption{Cold Gaussian start (with annealing on $\pi$)}
    \label{f:mixture_cold_annealing}
\end{figure}

\begin{figure}
    \centering
    \subfloat[$T = 5,000$]{%
        \includegraphics[width=\threefig]{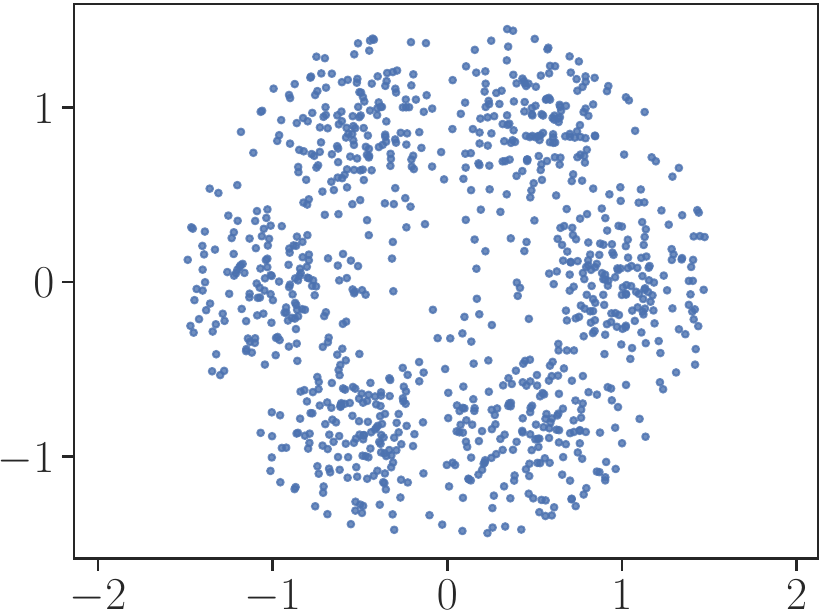}
    }
    \subfloat[$T = 10,000$]{
        \includegraphics[width=\threefig]{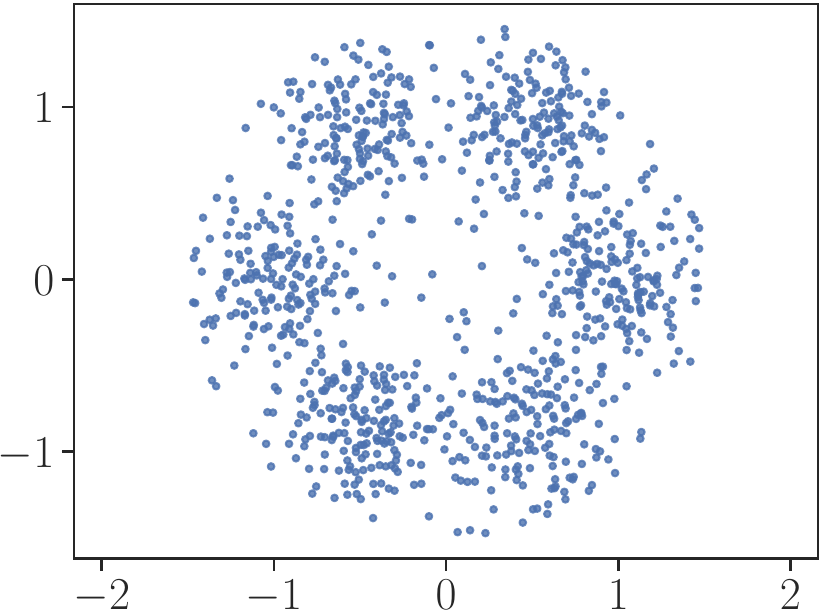}
    }
    \subfloat[$T = 15,000$]{
        \includegraphics[width=\threefig]{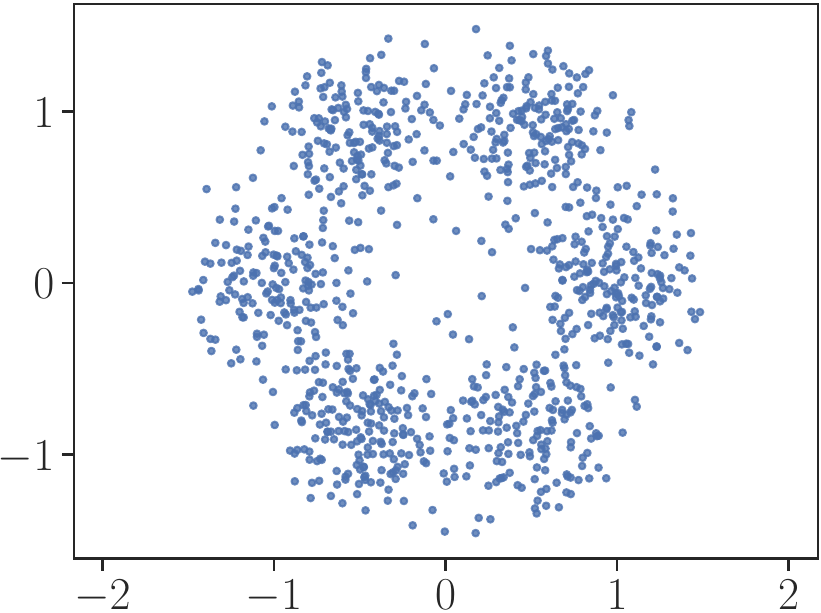}
    }
    \caption{Warm KSD descent start (with annealing on $\pi$)}
    \label{f:mixture_warm_ksd_annealing}
\end{figure}

It appears clear, as expected, that the warm start helps in a sense that fewer iterations are needed to reach a well spread distribution in each of the modes, even if a cold start from a Gaussian distribution remains quite reasonable in that case in terms of number of MALA iterations. 
In this toy model, the effect of annealing on $\pi$ does not seem tremendous.

A general remark is that we do not expect our method to perform well in cases where fast scaling algorithms as SVGD \citep{liu2016stein} or KSD descent \citep{korba2021kernel} fail.
 Another high level comment is that, in challenging
applications of Monte Carlo integration to the physical sciences, we rarely trust an MCMC chain to
find all the modes of a target. We usually repeatedly run a local optimization algorithm beforehand,
with a large number of random initial states, and cluster its output to find the modes of the target.
Then only can we validate that our MCMC chain has visited all modes a sufficient number of times.
Similarly, we only recommend using our algorithm after having heuristically identified all modes of
$\pi$.

\subsection{Decay of the energy for a non uniform target measure}
\label{a:energy_10d}

The setup in Figure~\ref{fig:decay_10d} is the same as in Figure~\ref{f:three_energies_var:energy_riesz} where we compare the energy (or MMD) between our procedure and MCMC, but this time for a non-uniform target measure.
We set $\pi$ to be a truncated Gaussian distribution on the unit ball in dimension $d = 10$ and keep $\beta_n = n^2$.
We observe the same results as in Figure~\ref{f:three_energies_var:energy_riesz}: the energy decay rate is the same as MCMC, but our approach improves on the MMD by about a factor $3$, even in this relatively high dimension ($d = 10$).
This reflects the fact that our $\PP$ is informed by the MMD.

\begin{figure}
  \centering
  \includegraphics[width=0.5\linewidth]{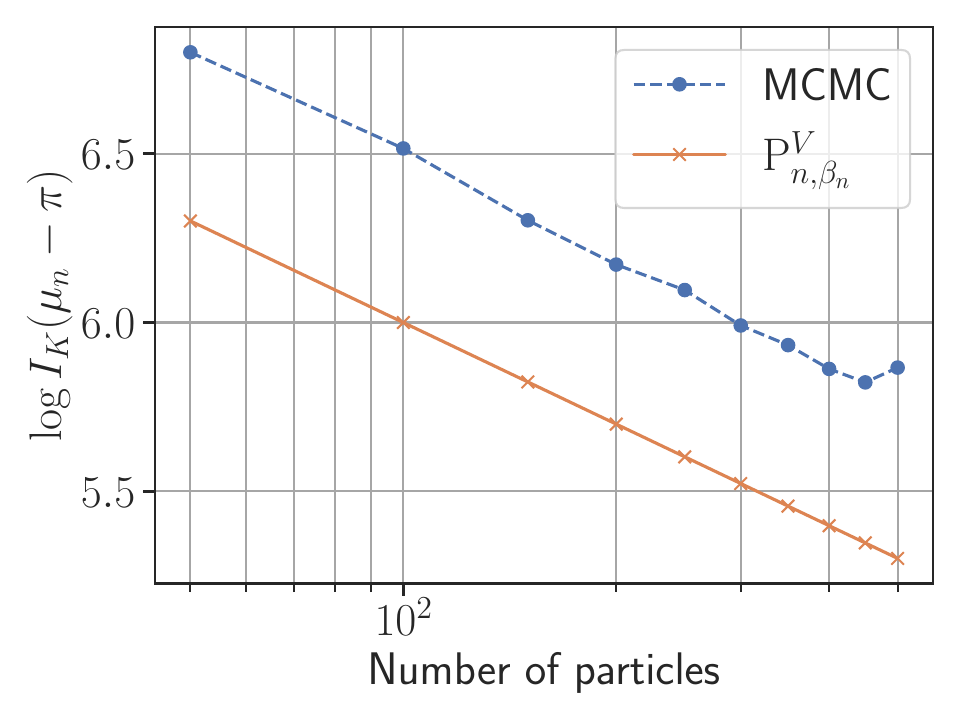}
  \caption{Energy comparison for truncated Riesz kernel in the same setup as Section 4 (Figure 2b), for a truncated Gaussian distribution $\pi$ on the unit ball in dimension $d = 10$.}
  \label{fig:decay_10d}
\end{figure}

\end{document}